\documentclass[11pt]{article}

\usepackage[final]{acl}

\usepackage{times}
\usepackage{amssymb}
\usepackage{threeparttable, tabularx}
\usepackage{booktabs}
\usepackage{orcidlink}
\newcommand{\up}[1]{\tiny ($\textcolor{green}{\blacktriangle}#1\%)$}

\usepackage{algorithm} 
\usepackage{algpseudocode}
\usepackage[]{algorithmicx}
\usepackage{mathtools}
\usepackage{amsthm}
\usepackage[textsize=tiny]{todonotes}

\newcommand{\cS}{\mathcal{S}}

\newcommand{\name}{\textsc{DEARICL}}
\newcommand{\cA}{\mathcal{A}}
\newcommand{\cE}{\mathcal{E}}
\usepackage{microtype}
\usepackage{graphicx}
\usepackage{tcolorbox}

\usepackage{subcaption}

\usepackage{amsmath}
\usepackage{pgfplots}
\usepackage{xcolor}
\usepackage{colortbl}
\definecolor{blue}{RGB}{17,220,247}
\definecolor{purple}{RGB}{163,115,250}
\definecolor{caribbeangreen}{rgb}{0.0, 0.8, 0.6}

\definecolor{bluecolor}{RGB}{0,0,255}

\definecolor{GREEN}{RGB}{84,130,53}

\newcommand{\colorg}{\cellcolor{gray!15}}
\usepackage{wrapfig}

\newcommand{\cV}{\mathcal{V}}

\usepackage{latexsym}
\pgfplotsset{compat=1.15}
\usepgfplotslibrary{
  statistics,
  colorbrewer,
  groupplots,
}
\usetikzlibrary{
  patterns,
  shapes.geometric,
  decorations.text,
  matrix,
  fit,
  backgrounds,
  positioning,
}
\tikzset{
  fignode/.style={
    outer sep=0.25em,
  }
}
\tikzset{
  framedfignode/.style={
    outer sep=0.25em,
    inner sep=0.5em,
    rounded corners,
    draw,
  }
}
\colorlet{plotColorNeutral}{gray}
\definecolor{plotColor1}{HTML}{f61a1c}
\definecolor{plotColor2}{HTML}{377eb8}
\definecolor{plotColor3}{HTML}{4daf4a}
\definecolor{plotColor4}{HTML}{984ea3}
\definecolor{plotColor5}{HTML}{FFFFCB}
\definecolor{plotColor6}{HTML}{1e90ff}
\colorlet{plotColorNeutral*}{plotColorNeutral!40}
\colorlet{plotColor1*}{plotColor1!60}
\colorlet{plotColor2*}{plotColor2!60}
\colorlet{plotColor3*}{plotColor3!60}
\colorlet{plotColor4*}{plotColor4!60}
\colorlet{plotColor5*}{plotColor5!60}
\colorlet{plotColor6*}{plotColor6!60}
\pgfplotsset{
    colormap={greenred}{HTML=(4daf4a) HTML=(e41a1c)},
    colormap={redgreen}{HTML=(e41a1c) HTML=(4daf4a)}
}

\newcommand{\EMPGAP}[3]{\hat{\Delta}_{#3}(#1, #2)}
\newcommand{\HA}[1]{\text{H}^{\varepsilon}(\rho)}

\newcommand{\namedyn}{\textsc{DearICL$_{dynamic}$}}
\newcommand{\namestatic}

\pgfplotsset{compat=1.15}
\usepgfplotslibrary{
  statistics,
  colorbrewer,
  groupplots,
}
\usetikzlibrary{
  patterns,
  shapes.geometric,
  decorations.text,
  matrix,
  fit,
  backgrounds,
  positioning,
}
\tikzset{
  fignode/.style={
    outer sep=0.25em,
  }
}
\tikzset{
  framedfignode/.style={
    outer sep=0.25em,
    inner sep=0.5em,
    rounded corners,
    draw,
  }
}
\colorlet{plotColorNeutral}{gray}
\definecolor{plotColor1}{HTML}{f61a1c}
\definecolor{plotColor2}{HTML}{377eb8}
\definecolor{plotColor3}{HTML}{4daf4a}
\definecolor{plotColor4}{HTML}{984ea3}
\definecolor{plotColor5}{HTML}{FFFFCB}
\definecolor{plotColor6}{HTML}{1e90ff}
\colorlet{plotColorNeutral*}{plotColorNeutral!40}
\colorlet{plotColor1*}{plotColor1!60}
\colorlet{plotColor2*}{plotColor2!60}
\colorlet{plotColor3*}{plotColor3!60}
\colorlet{plotColor4*}{plotColor4!60}
\colorlet{plotColor5*}{plotColor5!60}
\colorlet{plotColor6*}{plotColor6!60}
\pgfplotsset{
    colormap={greenred}{HTML=(4daf4a) HTML=(e41a1c)},
    colormap={redgreen}{HTML=(e41a1c) HTML=(4daf4a)}
}

\usepackage[T1]{fontenc}

\usepackage[utf8]{inputenc}

\usepackage{microtype}
\theoremstyle{plain}
\newtheorem{theorem}{Theorem}[section]

\newtheorem{lemma}[theorem]{Lemma}

\theoremstyle{definition}
\newtheorem{definition}[theorem]{Definition}

\theoremstyle{remark}

\usepackage{pgfplots}
\newcommand{\ARMS}{[K]}

\newcommand{\TOPM}{\cS_m^\star}

\newcommand{\EMPTOPTAU}{\mbox{$\hat{S}^{\tau_{\delta}}_m$}}
\newcommand{\EMPWORSTTAU}[1]{(\EMPTOPTAU)^c}
\newcommand{\EMPMU}[2]{\hat{\rho}_{#2}(#1)}

\usepackage{inconsolata}

\usepackage{graphicx}

\title{Data Efficient Sample Selection for In-Context Learning}

\author{  
 Venktesh V \orcidlink{0000-0001-5885-2175} \\ Stockholm University \\ \texttt{venktesh.viswanathan} \\\texttt{@dsv.su.se} \\
 Orcid: 0000-0001-5885-2175
\And Cem Levi \orcidlink{0009-0003-8367-101X} \\ SKIM Group B.V. \\ \texttt{c.levi}\\
\texttt{@skimgroup.com}\\
Orcid: 0009-0003-8367-101X
\And  Avishek Anand \orcidlink{0000-0002-0163-0739} \\ TU Delft \\ \texttt{avishek.anand}\\ \texttt{@tudelft.nl} \\
Orcid: 0000-0002-0163-0739}

\begin{document}
\maketitle
\begin{abstract}
The In-context learning (ICL) paradigm aids large language models (LLMs) to adapt to new tasks without need for fine-tuning. However, selecting an optimal combination of demonstration examples from a large pool of example subsets is a challenging problem. Existing approaches for selection do not model the complex relationship between ICL samples and downstream LLM performance. They typically perform static task-level selection, choosing subsets once offline, which can fail to generalize to unseen queries. We introduce \name{} (Data Efficient Algorithm for Ranking) ICL samples, a new framework that models demonstration example selection as a subset ranking problem. \name{} employs a non-linear surrogate employing a differentiable sorting objective within a gap-index bandit algorithm. The gap-index based approach enables fine-grained separation of good arms and borderline arms which is used as an auxiliary objective to train the non-linear surrogate through sufficient sampling of borderline arms, supporting instance-level subset ranking. On exemplar selection benchmarks with open-source LLMs, \name{} achieves \textbf{8.08–15.9\%} accuracy gains over strong linear bandit baselines, with low sample complexity. Code and data: \footnote{\url{https://github.com/VenkteshV/DearICL}}
\end{abstract}

\section{Introduction}

Selecting representative instances from a large pool is a recurring need across applications
such as training models, domain adaptation, and few-shot learning. In this work, we focus on \emph{in-context learning} (ICL) with large language models (LLMs), where LLMs can solve new tasks when
given a small sequence of demonstrations $(v,w)$ or $(v,\text{rationale},w)$ in context
\cite{brown2020gpt3}. However, naively choosing demonstrations (randomly or heuristically)
performs poorly~\cite{purohit-etal-2024-explora,li2023finding}. One principled way to
formulate this is as a \emph{subset selection} problem: choosing a small set of examples that
captures the task’s salient structure, according to a task-specific reward signal. The class of approaches that aim to select such near-optimal example subsets that achieves
the best average performance across all test / validation queries from a target task are broadly categorized as static (task-level) selection methods \cite{li2023finding,zhang-etal-2025-selecting}.

 Since, ICL performance
depends on \emph{subset interactions}, the joint effect of examples provided together, renders
a large search space of example subsets. 
For sample-efficiency, some of the existing task-level selection methods formulates subset selection as \emph{top-$m$ arm identification} in
multi-armed bandits (MABs), where each subset of examples is an arm and rewards are sampled based on LLM performance on the task~\cite{reda2021top}. Casting ICL demonstration choice as top-$m$ arm
identification provides a principled route to select representative subsets using the task-level rewards, while reducing number of LLM evaluations. Recent works ~\cite{case,purohit-etal-2024-explora}, builds upon gap-index bandit frameworks
like GIFA~\cite{reda2021top}, introducing challenger sampling mechanisms to handle the large
search space. While effective, these approaches have two major limitations. First, they rely
on a \emph{linear surrogate} to map arm features to rewards which cannot capture the complex dependencies between subsets and task
performance.


Secondly, a fundamental limitation of task-level (static) selection approaches is that they can
fail when new queries require skills absent from the fixed set of examples. Dynamic selection
chooses examples per test query, which improves flexibility but involves searching through a large space of
combinations and applying MAB algorithm per instance is computationally expensive. Existing approaches only employ heuristic semantic matching \cite{ye-etal-2023-complementary,rubin-etal-2022-learning} based on embeddings that are not adapted for the subset selection setting.  However, training ranking models for subset ranking that can clearly distinguish between optimal example subsets from sub-optimal ones for each test query is expensive. It requires sampling of sufficient number of borderline example subsets that are hard to distinguish for the subset ranker. 

Hence, it is critical to devise a sample-efficient mechanism that can help learning of subset ranking at instance-level which can also model the non-linear reward landscape of LLM based feedback. To overcome these limitations, we propose \name{}
(\textbf{D}ata \textbf{E}fficient \textbf{A}lgorithm for \textbf{R}anker training for demonstration samples selection in ICL), a new bandit framework for instance-level subset selection. \name{} is a gap-index MAB framework which focuses on top-m selection as auxiliary objective in the outer loop, while fitting a non-linear surrogate in the inner loop for instance-level subset ranking 
The top-m selection aims to implicitly induce a separation between near-optimal subsets (arms) and sub-optimal ones as discussed in literature \cite{maxgap} and hence the gap-index framework ensures sufficient number of borderline arms are sampled to induce this separation. Hence, by using this as auxiliary objective, our proposed approach, ensures that borderline subsets with rewards sampled at instance-level (each validation query) are used to fit the ranker in the inner loop. This enables the surrogate to clearly distinguish borderline example subsets from useful ones for each instance during inference for ranking example subsets for each test query.  We provide \textbf{theoretical guarantees}, including bounds on pairwise gap error and sample complexity. Our approach is sample-efficient
while improving performance (\textbf{8.08--15.9\%}) on \textbf{smaller LLMs}.

\section{Related Works}
\textbf{Exemplar Selection for ICL.}
The rise of LLMs has transformed them into general-purpose answering engines through emergent capabilities like ICL \citep{brown2020gpt3,wei2022emergent, wei2023chainofthought,wang-etal-2023-plan,kojima2023large,chen2022program} where a few examples are provided to LLMs to demonstrate the task.
To eliminate manual selection, several automated methods have emerged, such as reinforcement learning \citep{zhang-etal-2022-active,lu2023dynamic}, trained retrievers \cite{xiong2024dqlore}, Determinantal Point Processes \citep{ye2023compositional} and constrained optimization \citep{tonglet2023seer}. Additionally, instance-level selection methods that are learning-free, such as similarity-based \citep{rubin-etal-2022-learning}, complexity-based \citep{complex_cot}, and MMR \citep{ye-etal-2023-complementary}, have been explored. However, instance-level methods increase inference-time computational costs. To address this, a pre-selected, representative set of exemplars is chosen for ICL, akin to coreset selection methods \citep{guo2022deepcore}, though the key difference is that ICL does not involve parameter updates. While existing methods \cite{case, purohit-etal-2024-explora,bridge} propose bandit-based or bayesian optimization based task-level example selection algorithm, they focus on static selection, which does not transfer well to new test questions. \\
\textbf{Learning to Rank approaches:}The existing works from online learning to rank literature are related to our surrogate ~\cite{zoghi2017online,grotov2016online,li2019onlineltrs}, which learn the ranking models from user interaction or click data. However, our approach differs fundamentally, as we do not rely on direct user feedback but from LLM feedback, obviating challenges like de-biasing clicks.\\ 
\textbf{Top-m identification in stochastic bandits}:
The objective of top-$m$ arm identification is to identify those arms with highest means.
While fixed-confidence \citep{LUCB} and fixed-budget settings \citep{pmlr-v28-bubeck13} exist, our focus is the fixed-confidence setting, where the error probability to estimate the top-$m$ arms should be smaller than a predefined parameter $\delta \in (0,1)$. Adaptive sampling algorithms like UGapE \citep{UgapE} and LUCB \citep{LUCB}, along with uniform sampling methods \citep{pmlr-v30-Kaufmann13,pmlr-v54-chen17a}, have been introduced for the fixed confidence setup, but they lack efficiency in terms of sample complexity. 
While efficient adaptive sampling methods for linear bandits, such as  \cite{fiez2019sequentialexperimentaldesigntransductive}, RAGE \cite{zhang2023trainedtransformerslearnlinear}, LTS \cite{jedra2020optimal}, PEPS \cite{li2023optimalexplorationharderthompson}, LinGapE \citep{xu2017fullyadaptivealgorithmpure} and LinGame \cite{pmlr-v119-degenne20a},  have been proposed, they primarily address best-arm identification ($m=1$). GIFA \citep{reda2021top} was the first unified framework for efficient top-$m$ arm identification, but requires a significant number of gap-index computations and comparisons, leading to high sample complexity. 

\section{Methods}
\subsection{Problem Definition}
\label{sec:problem_definition}

In ICL, the model processes a sequence of input–output demonstrations followed by 
a new test input, and is expected to generate the corresponding output. 
We frame the problem of ICL as choosing 
representative example subset from a large dataset of existing examples~\cite{li2023finding} or training instances~\cite{2wikimultihopqa}. 
While static selection of these subsets can be modeled formally as a \emph{top-$m$ arm identification problem} in the multi-armed bandit framework due to it's ability to identify near optimal example subsets \cite{li2023finding,zhang-etal-2025-selecting}, we discuss the limitations of this static selection and demonstrate how a \textbf{new formulation} can be used as auxiliary objective for training non-linear surrogates as rankers dynamic selection.
Our framework primarily focuses on sample-efficient fitting of non-linear surrogate utilizing classical top-m selection in bandit literature as auxiliary objective. This surrogate is then employed for instance-level (dynamic) selection through subset ranking. 


\subsubsection{Subset Selection for ICL - Preliminaries and Limitations of Existing Formulation}
\label{sec:icl_subset_selection}

In ICL, a large language model processes a sequence of \emph{demonstrations} 
(i.e., input–output pairs) followed by a test input, and is expected to generate 
the corresponding test output. 
We denote each input by $v$ (e.g., a natural language question, a math problem, 
or a sentence to translate) and each output by $w$ (e.g., an answer, solution, 
or translation). 
In some tasks, $v$ may also contain additional reasoning such as rationales or 
chain-of-thought annotations; for example, solving math word problems may require 
showing intermediate steps, while translation tasks typically do not.

Let $\mathcal{Q} = \{(v_i, w_i)\}_{i=1}^{n}$ be a pool of $n$ candidate demonstrations, 
and let $(v_{test}, w_{test})$ denote a test instance. 
To perform inference on $v_{test}$, we select a subset $S \subseteq \mathcal{Q}$ 
of $k$ demonstrations and concatenate them with the test input to form the context
$E=\big[(v_{i_1}, w_{i_1}),., (v_{i_m}, w_{i_m}), v_{test}\big], \hat{w}_{test} = \mathbb{P}_{LLM}\!\left(\cdot \,\middle|\, E\right)$

The quality of the chosen subset $S$ has a direct impact on LLM performance. 
However, the candidate pool might comprise of good near-optimal subsets and sub-optimal subsets which could impact performance. 
This motivates the identification of the \emph{top-$m$ subsets} of demonstrations 
that are most useful for ICL to prune the sub-optimal subsets from the search space. 
Formally, let $\cS(\mathcal{Q})$ denote the set of all $k$-sized subsets of $\mathcal{Q}$, 
and let $\cV = \{(v_{val}(j), w_{val}(j))\}_{j=1}^{n'}$ be a validation set. 
The objective is to identify $\{a_1, \dots, a_m\} \subseteq \cS(\mathcal{Q})$ 
that maximize expected performance on $\cV$.

The inference procedure with top-$m$ subsets can then be written as:
$
E = \phi\!\left(\{a_1, \dots, a_m\}, v_{test}\right), 
\quad 
\hat{w}_{test} = \mathbb{P}_{LLM}\!\left(\cdot \,\middle|\, E\right),
$, where $\phi$ constructs the context from selected subsets. 
For task-level selection, $\phi$ may pick the subset with the lowest validation loss on $\cV$. 


Nevertheless, static subsets selected offline as described above may fail to generalize to unseen queries, 
while fully online selection is computationally prohibitive since running bandit algorithms 
to convergence per test instance incurs high latency. 
This motivates the need for a principled formulation that is both \emph{expressive}, 
capturing complex dependencies in LLM rewards, and \emph{efficient}, supporting 
practical inference-time selection. We cast it as an auxiliary objective due to it's ability to sample borderline arms to adaptively train a non-linear surrogate in the internal loop.
\subsubsection{Top-$m$ Arm Selection as Auxiliary objective}
\label{sec:topm_arm_selection}

We present the \emph{top-$m$ arm identification task} 
in a multi-armed bandit (MAB) setting. Each candidate subset $a_i \in \cS(\mathcal{Q})$ 
is treated as an arm, with reward defined by LLM performance on $\cV$. 
Since $|\cS(\mathcal{Q})|$ can be large, the search space 
$\cS(\mathcal{Q})^m$ of all possible $m$-subsets is extremely challenging.

Let the true reward for an arm $a$ be
$
\rho(a; \theta^*) = \mathcal{F}_{\theta^*}(x_a),
$
where $x_a \in \mathbb{R}^n$ are arm features and $\theta^*$ are the true parameters. 
Each evaluation of an arm yields a noisy observation:
\[
\hat{\rho}(a; \theta) = \rho(a; \theta) + \eta, 
\quad 
\mathbb{E}[e^{\lambda \eta}] \leq \exp\!\left(\tfrac{\lambda^2 \chi^2}{2}\right),
\]
where $\eta$ is sub-Gaussian with variance $\chi^2$. The top-$m$ identification objective is to output $\widehat{\mathcal{S}}_m$ such that
$
\mathbb{P}\!\left(\widehat{\mathcal{S}}_m \neq \mathcal{S}^\star_m\right) \le \delta,
\quad 
\mathcal{S}^\star_m = \{1,2,\dots,m\},
$

\textbf{Gap-index methods and their limitations.} 
Gap-index bandit algorithms such as GIFA~\cite{reda2021top} address top-$m$ arm 
identification by iteratively estimating arm parameters and comparing the most 
\emph{ambiguous arms} using gap indices.
However, GIFA relies on a \emph{linear surrogate} $\mathcal{F}_\theta$, 
which is restrictive when modeling rewards induced by LLMs. Linear surrogates cannot 
capture complex, non-linear dependencies between subsets and task performance. 
Additionally, \textbf{simply substituting a non-linear surrogate in existing gap-index frameworks is insufficient} for our problem setup.  In \name{}, we propose a novel
gap-index bandits framework with a non-linear ranking surrogate, with theoretical bounds on pairwise gap error and sample complexity which applies to the auxiliary top-m selection objective and also number of arm samplings required to fit the non-linear surrogate. After convergence on top-m selection objective, the non-linear surrogate can be employed for dynamic subset selection at inference time. New confidence widths and index computations are devised in the \emph{gap-index framework} to incorporate the \textbf{non-linear surrogate} within the \textbf{auxiliary objective} of separating optimal arms from sub-optimal ones (top-m selection) and to ensure convergence. 
\subsection{\name{}: A sample-efficient algorithm fitting of non-linear Ranking surrogate}
\label{sec:algorithm}

Based on above discussion of the problem setup we propose a gap-index based bandit algorithm \name{} with a non-linear ranking surrogate. The top-m identification algorithm implicitly induces a clustering to separate top-m near-optimal arms from other borderline or sub-optimal arms \cite{maxgap}. Hence, the top-m identification is used as \textit{auxiliary objective} to fit the non-linear surrogate, so that it samples sufficient number of borderline arms apart from just repeatedly sampling good arms. Since, the task can also be viewed as learning to rank (LTR) the example subsets (arms), we adopt a \textit{differentiable sorting} objective \cite{pirank} to train the non-linear surrogate. This surrogate learns to approximate the rewards of arms during the offline run of the bandit algorithm and hence can be used during runtime to rank example subsets for dynamic selection. At each step, an exemplar subset (arm) is regarded as a document whose current empirical mean is estimated by $\mathcal{F_{\theta}}$. $\mathcal{F_{\theta}}$ is a multi-layer connected network architecture with RELU activations.  
 \begin{align}
\hat{\rho}(a_i) =
\frac{1}{n^{'}} \sum_{i=1}^{n^{'}}(\mathcal{F}_{\theta}(x_a(v_{val}(i)), v_{val}(i))) , 
\label{eq:empirical_mean}
\end{align}
\[
x_a(v_{val}(i)) = \Biggl[ \, \mathcal{H}\bigl(v_{val}(i)\bigr), \; 
\frac{1}{k} \sum_{l=1}^{k} \mathcal{H}\bigl(v_{i_l}; w_{i_l}\bigr) \, \Biggr]\]

 where $v_{val}(i)$ from $\cV$ is treated as a query and the example subset (arm) is treated as the document. We average the sentence embeddings of examples obtained using an encoder $\mathcal{H}$ in the arm to provide a single feature representation for the arm $x_a$ of dimension $d$ with respect to query $v_{val}(i)$. The representation with query embedding is input to the surrogate $\mathcal{F_{\theta}}$ to obtain empirical mean estimate.
 \begin{algorithm}[hbt!]
\small
\caption{\textsc{\name{}}}
\label{alg:grass}
\begin{algorithmic}[1]
\small
\Require 
  $\mathcal{Q}$ (training exemplars);
  $k$ (prompt size);
  $\mathcal{S}$ (all $k$-subsets of $\mathcal{Q}$);
  $m$ (top-m :~target); $m'$ (challenger size);
  $\mathcal{G}$ (data generator);
  $N$ (query size);
  $\mathcal{F}_\theta$ (surrogate model)
\Ensure 
  $\mathcal{F}_{\theta}$ (trained surrogate) and estimated top-$m$ set

\vspace{0.5em}
\State \textcolor{blue}{{Initialize}}
\State \quad $U_0 \gets$ random $m$ arms from $\mathcal{S}$ \Comment{Current Top-m}
\State \quad $C_0 \gets$ the next best $m'$ arms where $m' < m$ (is resampled every iteration) \Comment{Challengers}
\State \quad $\mathcal{D} \gets$ 0  or random $M$ subsets from $\mathcal{G}(\mathcal{Q})$ to solve cold-start problem
\State \quad $t \gets 1$

\vspace{0.5em}
\While{$\neg$ ($B_t(ch_t,b_t) \le \epsilon$) }   

\vspace{0.5em}
  \State \textcolor{blue}{\textsc{(a) Identify arms to swap}}
  \State \quad $m_t \gets \arg\min_{a \in U_{t-1}} \hat{\rho}_t(a)$ \Comment{Weakest in $U_t$}
  \State \quad $c_t \gets \arg\max_{a \in C_{t-1}} \hat{\rho}_t(a)$ \Comment{Strongest challenger}

\vspace{0.5em}
  \State \textcolor{blue}{{(b) Border Update (Swap if Needed)}}
  
  \If{$\hat{\rho}_t(c_t) \ge \hat{\rho}_t(m_t) $}
    \State \quad Exchange $m_t$ and $c_t$ between $U_{t-1}$ and $C_{t-1}$
  \EndIf
  \State \quad $U_t, C_t \gets$ updated sets

\vspace{0.5em}
  \State \textcolor{blue}{(c) Expand Candidate Pool}
  \State \quad $M_t \gets$ random $m'$ arms from $(U_t \cup C_{t-1})^c$
  \State \quad $C_t \gets \text{top}_{m'}(M_t \cup C_{t-1}; \hat{\rho}_t)$

\vspace{0.5em}
  \State \textcolor{blue}{\textsc{(d) Recompute Ambiguity Frontier}}
  \State \quad $b_{t+1} \gets \arg\max_{b \in U_t}\ \max_{ch \in C_t} B_t(ch,b)$
  \State \quad $ch_{t+1} \gets \arg\max_{ch \in C_t}\ B_t(ch, b_{t+1})$

\vspace{0.5em}
  \State \textcolor{blue}{\textsc{(e) Acquire Feedback \& Retrain}}
  \State \quad $a_{t+1} \gets \text{selection\_rule}(U_t, C_t)$
  \State \quad $r_{t+1}(a_{t+1})\gets R(\psi(a_{t+1}),\mathcal{V})$\Comment{LLM calls}
  \State \quad $\mathcal{D} \gets \mathcal{D} \cup \mathcal{G}(r_{t+1}(a_{t+1}), a_{t+1})$
  \State \quad Retrain $\mathcal{F}_\theta$ for one epoch on $\mathcal{D}$

  \State $t \gets t+1$
\EndWhile

\State \textbf{Return} $\mathcal{F}_{\theta}$
\end{algorithmic}
\end{algorithm}

 Then the arm (example subset) being played provides the score based on LLM output on multiple validation samples from $\cV$ as rewards. For a single val. sample, reward is
\begin{equation}
\small
\mathcal{R}\!\left( \psi(a_i), \, v_{val}(i) \right) 
= \gamma \Biggl( \, \mathbb{P}_{LLM}\!\left( \cdot \,\middle|\, \psi(a_i, \, v_{val}(i))\right) \Biggr)
\end{equation}
Here $\gamma$ could indicate accuracy or other relevance measures like BertScore which compares the generated output from LLM $\hat w_{val}$ with ground truth $w_{val}$ and outputs a relevance score. And $\psi$ denotes the context / prompt generator function based on given subset of examples and the query to be answered.
 Hence reward for an arm can be obtained as:  $ r(a_i) = \mathcal{R}(\psi(a_i),\cV) =  [\mathcal{R}(\psi(a_i),v_{val}(i))..\mathcal{R}(\psi(a_i),v_{val}(n^{'}))]$ which is then added as relevance labels of  subset (arm) with respect to validation samples (queries) to the dataset $D$, used to fit the non-linear surrogate to correct estimation errors with ranking loss (Step 24, 25 Alg. \ref{alg:grass}) where $\mathcal{G}$ is the data formatter,loader.
 
An overview of the gap-index based bandit algorithm with differentiable sorting surrogate is shown in Algorithm \ref{alg:grass}. First a shortlist of good arms $U_0$ is initialized to random $m$ arms in Steps 1-3. The dataset $\mathcal{D}$ which is used to fit the non-linear ranking surrogate is initialized to empty set or random data to solve the cold start problem (Line 4). The updated $U_t$ is computed by selecting the worst-arm in $U_t$ with lowest empirical mean in current step  and swap it with the best challenger arm $c_t$ in the challenger shortlist $C_t$ (Lines 8-14). The empirical means for above steps are computed using the formulation in Equation \ref{eq:empirical_mean}. In Lines 15-17, we uniformly sample $m'$ arms from $(U_t\cup C_{t-1})^c$, to generate the set $M_t$, and then select the top-$m'$ arms from $M_t\cup C_{t-1}$ to generate the updated $C_t$.
 The most ambiguous arms $b_t$ (guess for m-best arm) and $ch_t$ (a potentially misassessed arm m-best arm) which determine the stopping criterion are computed with help of gap-indices as shown in Steps 18-20. The gap-index between any two arms $i,j$ is computed as: $B_t(i,j)=\hat{\rho}_t(i)-\hat{\rho}_t({j})+W_t(i,j)$.  Here in gap-index computation, $W_t(i,j)$ is computed as per \textbf{Equation \ref{eq:w_t}} ( the RHS of the inequality from \textbf{Theorem \ref{theorem1}}) accounting for the non-linear ranking based surrogate in \name{}. $C_{t}\cup U_t$ bounds the amount of comparisons required for gap-index computations unlike GIFA. The intuition here is that once the gap-index between most ambiguous (borderline) arms approaches $\epsilon$ (Line 6), there is no confusion  between the  empirically estimated top-$m$ arms and it's closest competitor in the challenger shortlist.  Then we employ a greedy selection rule \cite{reda2021top}, where the arm that minimizes the variance between $b_t$ and $ch_t$ is selected.  The error in the empirical mean estimates with respect to rewards at instance-level ( validation sample) for the non-linear surrogate is computed by the version of loss $\mathcal{L}(\vec{\hat \rho},\vec{r}) = -\widehat{\mathrm{NDCG}}(\vec{\hat \rho},\vec{r})
$ employing a relaxed version of NDCG metric \cite{pirank}. Here $\vec{r}$ and $\vec{\hat \rho}$ are rewards and empirical mean vectors (across arms equivalent to documents) for a query (validation sample) from $D$.   Then through a Stochastic Gradient Descent (SGD) step the surrogate is updated to better estimate the empirical mean ((Lines 22-25)). Hence, using the top-$m$ subsets (arms) selection formulation, which aims to clearly separate near-optimal arms from sub-optimal ones, as an auxiliary objective, our approach provides an \textbf{efficient mechanism} to learn a differentiable sorting model for instance-level subset ranking objective based on LLM feedback.

\subsection{Sample Complexity Theoretical Bounds }
\label{sec:sample_complexity_proof}
Following \cite{reda2021top}, we obtain a high probability ($1-\delta$) upper bound on sample complexity of \name{} which is non-trivial and different from linear MAB variants. 
\begin{definition} (Good Gap indices)
    $\cE \triangleq \bigcap_{t > 0} \bigcap_{i,j \in \ARMS} \Big(\rho_i-\rho_j \in [-B_t(j,i), B_t(i,j)]\Big),$
\end{definition}

with $\mathbb{P}(\cE) \ge 1-\delta$ which denotes that a good choice of gap indices $B_t(i,j)$ satisfies event $\cE$ with probability greater than or equal to  $1-\delta$. For the above event to hold, it is essential to prove the following bound on pairwise-gap error

\begin{figure*}[!t]
\small
\begin{subfigure}{0.3\linewidth}
\begin{tikzpicture}
\begin{axis}[
    ylabel style = {font= \tiny},
    xlabel style = {font= \small},
    xlabel=Rounds,
    xticklabel style={font=\small},
    yticklabel style={font=\boldmath \tiny},
    ylabel=Gap Index,
    height=4.3cm,
    legend style={
                    font= \tiny,
                },
    width=4cm,
    xmin=1, xmax=80,
    ymin=0, ymax=1.2]
    ytick={0,0.1,0.2,0.3}
    xtick={1,10,20,30,40...,80}
\addplot[smooth,mark=*,red] plot coordinates {
    (1,1.18)
    (10,0.51)
    (20,0.43)
    (30,0.32)
    (40,0.30)
    (50,0.27)
    (60,0.19)
    (70,0.19)
    (79,0.05)};
\addlegendentry{\name{}}
\addplot[smooth,mark=*,blue] plot coordinates {
    (1,0.87)
    (10,0.50)
    (20,0.44)
    (30,0.40)
    (40,0.38)
    (50,0.33)
    (60,0.31) 
    (70,0.23)
    (80,0.20)};
\addlegendentry{CASE}
\end{axis}
    \end{tikzpicture}
\caption{Gap Index \\Comparison}
\label{fig:gap_index}
\end{subfigure}
\hspace{-3em}
\begin{subfigure}{0.3\linewidth}
\begin{tikzpicture}
\begin{axis}[
    ylabel style = {font= \tiny},
    xlabel style = {font= \small},
    xlabel=Rounds,
    xticklabel style={font=\small},
    yticklabel style={font=\boldmath \tiny},
    ylabel=Simple Regret,
    height=4.3cm,
    legend style={
                    font= \tiny,
                },
    width=4cm,
    xmin=1, xmax=80,
    ymin=0, ymax=1]
    ytick={0,0.1,0.2,0.3}
    xtick={1,10,20,30,40...,80}
\addplot[smooth,mark=*,red] plot coordinates {
    (1,0.45)
    (10,0.41)
    (20,0.39)
    (30,0.32)
    (40,0.25)
    (50,0.21)
    (60,0.16)
    (70,0.12)
    (80,0.08)};
\addlegendentry{\name{}}
\addplot[smooth,mark=*,blue] plot coordinates {
    (1,0.90)
    (10,0.79)
    (20,0.54)
    (30,0.42)
    (40,0.33)
    (50,0.27)
    (60,0.22) 
    (70,0.14)
    (80,0.14 )};
\addlegendentry{CASE}
\end{axis}
    \end{tikzpicture}
\caption{Simple Regret \\Comparison}
\label{fig:simple_regret}
\end{subfigure}
\hspace{-3.5em}
\begin{subfigure}{0.25\linewidth}
\begin{tikzpicture}
\edef\mylst{"","","","",""}
\edef\case{"","","","",""}

    \begin{axis}[
            ybar=5pt,
    height=4.3cm,
    width=4.3cm,
            bar width=0.25,
            every axis plot/.append style={fill},
            grid=major,
            xtick={1, 2, 3},
            xticklabels={GSM, Aqua, WMT19},
            xlabel={Datasets},
            ylabel style = {font=\tiny},
        yticklabel style = {font=\boldmath \tiny,xshift=0.5ex},
        xticklabel style ={font=\tiny,yshift=0.5ex},
            ylabel={Time (in \textbf{hours})},
            enlarge x limits=0.25,
            ymin=0,
            ymax=18,
            legend style ={font=\tiny,yshift=0.5ex},
            area legend,
            nodes near coords style={font=\tiny,align=center,text width=2em},
            legend entries={\name{}, CASE},
            legend cell align={left},
            legend pos=north east,
            legend style={/tikz/every even column/.append style={column sep=0.5cm}},
        ]
        \addplot+[
            ybar,
            plotColor3*,
            nodes near coords=\pgfmathsetmacro{\mystring}{{\mylst}[\coordindex]}\textbf{\mystring},
            nodes near coords align={vertical},
            draw=black,
            postaction={
                    pattern=north east lines
                },
        ] plot coordinates {
                (1,7.22)
                (2,4.38)
                (3,1.61)
            };
        \addplot+[
            ybar,
            plotColor1*,
            draw=black,
            nodes near coords=\pgfmathsetmacro{\mystring}{{\case}[\coordindex]}\textbf{\mystring},
    nodes near coords align={vertical},
            postaction={
                    pattern=north west lines
                },
        ] plot coordinates {
                (1,17)
                (2,7.93)
                (3,2.9)
            };

    \end{axis}

\end{tikzpicture}
\centering
\caption{Subset Selection \\ runtime Comparison}
\label{fig:runtime_comparison}
\end{subfigure}
\hspace{-0.1em}
\begin{subfigure}{0.25\linewidth}
\begin{tikzpicture}
\edef\mylst{"","","","",""}
\edef\case{"","","","",""}

    \begin{axis}[
            ybar=5pt,
    height=4.3cm,
    width=4.3cm,
            bar width=0.25,
            every axis plot/.append style={fill},
            grid=major,
            xtick={1, 2, 3},
            xticklabels={GSM, Aqua, WMT19},
            xlabel={Datasets},
            ylabel style = {font=\tiny},
        yticklabel style = {font=\boldmath \tiny,xshift=0.5ex},
        xticklabel style ={font=\tiny,yshift=0.5ex},
            ylabel={Time (in \textbf{seconds})},
            enlarge x limits=0.25,
            ymin=0,
            ymax=20,
            legend style ={font=\tiny,yshift=0.5ex},
            area legend,
            nodes near coords style={font=\tiny,align=center,text width=2em},
            legend entries={\name{}, CASE},
            legend cell align={left},
            legend pos=north east,
            legend style={/tikz/every even column/.append style={column sep=0.5cm}},
        ]
        \addplot+[
            ybar,
            plotColor3*,
            nodes near coords=\pgfmathsetmacro{\mystring}{{\mylst}[\coordindex]}\textbf{\mystring},
            nodes near coords align={vertical},
            draw=black,
            postaction={
                    pattern=north east lines
                },
        ] plot coordinates {
                (1,6.85)
                (2,11.79)
                (3,1.68)
            };
        \addplot+[
            ybar,
            plotColor1*,
            draw=black,
            nodes near coords=\pgfmathsetmacro{\mystring}{{\case}[\coordindex]}\textbf{\mystring},
    nodes near coords align={vertical},
            postaction={
                    pattern=north west lines
                },
        ] plot coordinates {
                (1,6.04)
                (2,11.93)
                (3,1.57)
            };

    \end{axis}

\end{tikzpicture}

\caption{Avg. inference \\runtime/query}
\label{fig:inference_runtime_comparison}
\end{subfigure}
\caption{Top-$m$ arm identification by \name{}, CASE for AquaRAT.   (a) Gap Index ($B_t(ch_t,b_t)$) comparison and (b) Simple regret comparison, (c) subset selection time,  (d) Average inference time/query.}
\label{fig:comparison_time_20}
\end{figure*}
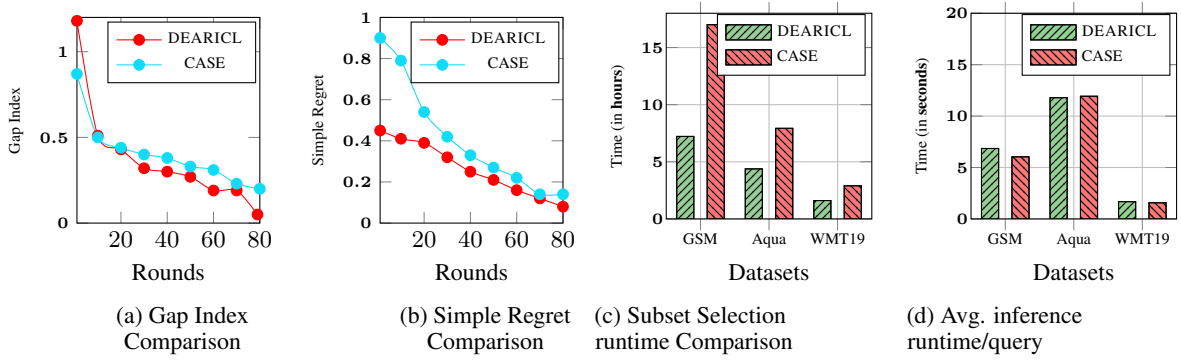
\begin{table*}[h!]
     \small

    \centering
    
    \begin{tabular}{llll}

    \toprule
     \textbf{Method}& {\textbf{GSM8K}}& \textbf{AquaRat} & \textbf{WMT19} 
     \\

    \midrule
         \colorg \textbf{Task level} & \colorg & \colorg & \colorg  \\
Zero-shot-COT \cite{kojima2023large} & 37.37 & 36.61 & 51.22  \\
          Few-Shot COT \cite{wei2023chainofthought} & 63.22 & 40.94 & 61.64 \\
         LENS \cite{li2023finding} & 64.97 & 44.88 & 64.09 \\

         EXPLORA \cite{purohit-etal-2024-explora} & 69.92&47.24 & 66.26 \\
         BRIDGE \cite{bridge} & 64.52 & 48.81 & 64.96 \\
    Static CASE \cite{case}    &        67.00              &  46.06  & 65.40  \\

    \textbf{\name{}}$_{static}$ (ours) & 69.90 & 50.00 & 68.83\\

\colorg \textbf{Instance Level } & \colorg & \colorg & \colorg  \\
    KNN \cite{rubin-etal-2022-learning} & 61.07 & 41.31 & 68.36 \\

    MMR \cite{ye-etal-2023-complementary} & 66.48& 45.84 & 68.86 \\

    PiRank \cite{pirank}      & 69.30     & 37.00      & 71.30   \\
    NeuralNDCG \cite{NeuralNDCG}  & 72.23     & 44.09      & 71.13             \\
    ListNet \cite{listnet,Pobrotyn2020ContextAwareLT}     & 70.30       & 43.30     & 71.26            \\
    UDR \cite{li-etal-2023-unified}     & 67.30       & 44.09     & 71.22            \\
    NDCGLoss 2++ \cite{LambdaLoss}      & 70.00       & 43.70     & 72.56            \\
    NeuralNDCG with Normalized Data      & 71.33       & 45.66     & 71.80            \\

       \colorg \textbf{Instance level (Bandit Approaches)} & \colorg & \colorg & \colorg  \\

     $CASE_{dynamic}$ \cite{case}   & 70.00              &  47.51  & 71.70    \\
    
       \colorg \textbf{Instance level (Bandit + LTR)} & \colorg & \colorg & \colorg  \\

        \textbf{\namedyn{}} (ours) &     \textbf{75.66\up{8.08}}$\dagger$  & \textbf{55.11\up{15.99}}$\dagger$      & \textbf{78.89\up{10.03}}$\dagger$    \\
        \textbf{\namedyn{} (-exploration) (ablation)}  &   70.30    & 37.00      & 71.30   \\
        \textbf{\namedyn{} (random) (ablation)}  &   69.06    & 49.21      & 70.38   \\

     \bottomrule
    \end{tabular}


   \caption{Results using \texttt{llama3.2:3b}$\dagger$ indicates statistical significance (t-test) over $CASE_{dynamic}$ at 0.05 level. We report performance using the official metrics: Exact Match (EM) (AquaRAT,GSM8K) and BertScore \cite{bertscore} (WMT19) for the respective datasets.}
       \label{tab:ltr-dynamic-perf}
\end{table*}
\begin{theorem}
    In a fixed-confidence setting,  $\delta \in (0,1)$, with probability at least $1-\delta$, 
for all pairs $i,j \in \mathcal{A}$:
\vspace{-1em}
\begin{align}
\big| (\hat{\rho}_t(i) - \hat{\rho}_t(j)) - (\rho(i)-\rho(j))\big|
\le \notag \\ c_t \sqrt{ \tfrac{2 \widehat V_{ij,t} \log(2K^2/\delta)}{N} } 
+ \varepsilon^{\text{stab}}_t + b_i + b_j  \notag \\ + \tfrac{4M \log(2K^2/\delta)}{3N}.
\label{eq:w_t}
\end{align}

where $\widehat V_{ij,t}$ is the empirical variance of MC-dropout differences, for $N$ stochastic predictions (MC dropout forward passes) using set $\cV$ $\bar d_{ij}:=\frac{1}{N}\sum_{k=1}^N\big(y_i^{(k)}-y_j^{(k)}\big),
\widehat V_{ij,t}:=\frac{1}{N}\sum_{k=1}^N\big(y_i^{(k)}-y_j^{(k)}-\bar d_{ij}\big)^2.$,
$\varepsilon^{\mathrm{stab}}_t$ is the SGD stability error, and $b_i,b_j$ are
surrogate approximation biases.
\label{theorem1}
\end{theorem}

\noindent \textbf{Theorem \ref{theorem1} Proof structure}:  Appendix  \ref{sec:grass_proof}.


\begin{theorem}\label{th:upper_bounds_linear_topm}
\textbf{Sample complexity}: For \name{}, 
on event $\cE$ on which the algorithm is ($\varepsilon, K, \delta$)-PAC, stopping time $\tau_{\delta}$ satisfies $\tau_{\delta} \leq \inf \{u \in \rho^{*+}: u > 1+\HA{} \frac{log(2 K^2 / \delta)}{N} + \mathcal{O}(K)\}$, where, for algorithm  with the largest variance selection rule \footnote{or pulling both arms in $\{b_t,c_t\}$ at time $t$}~: 
$\HA{\cA} \triangleq 18 
c_t^2  \sum_{a\in [K]} \sigma_{a,t}^2\cdot
\max\!\Big\{ \varepsilon^{-2},\; \big(\tfrac{\varepsilon+\Delta(a)}{3}\big)^{-2} \Big\},$.
\end{theorem}
\textbf{Theorem 3.3 Proof}:  On event $\cE$, we first demonstrate that the Lemma \ref{lemma:m-GRASS_bound} below holds. Then using stopping criterion and Lemma \ref{lemma:m-GRASS_bound} we derive the upper bound on sample complexity. The detailed proof is available in Appendix \ref{sample_complexity}. The bound holds for arms in $U_T\cup C_T$. It implies that the  top-$m$ arms from $U_T\cup C_T$ are present in $U_T$ with prob. $1-\delta$, if $T> \tau_{\delta}$, and $K$ is the size of $U_T\cup C_T$.

\begin{lemma}\label{lemma:m-GRASS_bound}
On the event $\cE$, for all $t > 0$, 
\begin{align*}
    B_t(ch_t, b_t)(t) \leq \min(-({\Delta(b_t)}  \lor \notag \\{\Delta(ch_t)}) 
+2W_t(b_t,ch_t), 0)+W_t(b_t,ch_t)
\end{align*} 

\end{lemma}

, where $a \lor b= max(a,b)$. In summary, Theorem \ref{theorem1} helps support Definition 3.1.  Theorem \ref{theorem1} establishes a uniform high-probability bound on pairwise gap estimation error, ensuring that the empirical gap between any two arms concentrates around the true gap. This directly justifies Definition 3.1, which defines the “good-gap” event $\cE$ under which all pairwise comparisons are well-behaved. Since event $\cE$ follows from Theorem \ref{theorem1}, conditioned on event 
$\cE$, Lemma \ref{lemma:m-GRASS_bound} characterizes how the algorithm’s adaptive confidence radii shrink over time, and Theorem \ref{th:upper_bounds_linear_topm} then converts these shrinkage properties into a high-probability sample-complexity upper bound for \name{} which represents the expected number of LLM calls for convergence of auxiliary objective in \name{}.
\begin{table}[hbt!]
\small
\centering
\begin{tabular}{lrrr}
\toprule
\textbf{Method} & \textbf{GSM8K} & \textbf{AquaRat} & \textbf{WMT19} \\
\midrule
$CASE_{dynamic}$        & 1554.19 & 2378.04 & 106.54 \\
Neural NDCG        & 10.59   & 18.51   & 5.20   \\
PiRank             & 6.10    & 11.59   & 1.80   \\
$\namedyn{}$ & 6.04 & 11.59 & 1.68 \\
\bottomrule
\end{tabular}
\caption{Inference Time Comparison (in seconds)}
\label{tab:inference_time}
\end{table}

\section{Experimental Setup}
\label{sec:exp}
We aim to answer the following research questions:

 \textbf{RQ1}: Does \name{} with auxiliary objective converge faster than state-of-the-art approaches for top-m ICL sample subsets identification?
 
 \textbf{RQ2}: Does instance-level selection using surrogate from \name{} with auxiliary objective lead to improved downstream task performance? 
 
\textbf{RQ3}: Can \name{} lead to improved task performance without sacrificing efficiency ?
\subsection{Experimental Setup}
\textbf{Datasets and Metrics:}  For mathematical reasoning, we use GSM8K and AquaRAT. For demonstrating generalization, we also evaluate on a translation task WMT 2019. Detailed descriptions of the datasets are provided in Appendix \ref{sec:datasets}. We report performance using the official metrics: Exact Match (EM) (AquaRAT,GSM8K) and BertScore \cite{bertscore} (WMT19) for the respective datasets. For \textbf{reward (LLM feedback)}, we compute BertScore \cite{bertscore} between generated and ground truth rationales, answers for GSM8K and AquaRAT, between translations for WMT19 following \cite{jauregi-unanue-etal-2021-berttune}.\\ 
\textbf{LLMs and hyperparameters:} We primarily evaluate on relatively stable open-source LLMs like Llama3.2-3b. We also report performance on closed source models like gpt-4o-mini in \textbf{Appendix \ref{app:alternative_llm}} and \textcolor{black}{ results using  open models like Deepseek-R1:7B (DeepSeek-R1-Distill-Qwen-7B) are shown in \textbf{Table \ref{tab:deepseek}}}. For all approaches we set max\_tokens to predict to 1000 with temperature of 0.25. 
 
\textbf{Baselines}: We compare with existing state-of-the-art example selection algorithms like CASE \cite{case}, BRIDGE \cite{bridge}, EXPLORA \cite{purohit-etal-2024-explora} and since our surrogate is based on LTR philosophy we compare with LTR baselines. The number of top arms to be identified was set as $m=|U_t|=10$ and $|C_t|=5$. The confidence parameter was fixed at $\delta=0.05$, controlling the probability of incorrectly identifying the top-m arms. The stopping criterion which is the gap between $U_t$ and $N_t$ was also kept at $\varepsilon=0.1$. The example subsets ($\cS$) are formed, by sampling \textit{with replacement} one example from each of the 5 clusters formed from training set. We use the same hyperparameters in MAB setup for \name{} for fair comparison. \textbf{Dynamic CASE} ($CASE_{dynamic}$ ) method is obtained by applying CASE \cite{case}, for each test instance instead of single offline run. For instance-level selection using the surrogate ($\name{}_{dynamic}$) the search space is union of all $U_t$ and $C_t$ shortlists from all rounds of the bandit run.
\textbf{Learning to Rank baselines} - We compare to diverse LTR approaches including PiRank \cite{pirank} and learned rankers for ICL like UDR \cite{li-etal-2023-unified}. The same hyperparameter values were also used for all LTR baselines and for \name{}.
Optimization was performed using Adam with learning rate of $1e-4$. The training ran for 100 epochs with a batch size of 16. LTR baselines hyperparameters are in Table \ref{tab:ltr_params}.


\section{Results}
\vspace{-0.7em}
\subsection{Empirical Verification of Convergence in \name{} }

\label{sec:regret_gap}

To answer \textbf{RQ1}, we record, compare and analyze the gap-index and simple regret across rounds. Since the stopping criterion is directly dependent on gap-index, it should decrease across rounds with minor fluctuations for convergence. The gap-index across rounds is compared across different algorithms as shown in Figure \ref{fig:gap_index}. We observe that for \name{}, gap-index decreases gradually and approaches 0 demonstrating that our proposed non-linear ranking surrogate based MAB algorithm converges demonstrating it's correctness empirically. We also observe that it converges earlier than CASE. While CASE initially shows a monotonically decreasing trend in the  gap, it stagnates after round 70 struggling to approach $\epsilon$. We observe that this is primarily due to the model struggling to distinguish between truly good arms and  borderline challenger arms that appear to be good.  We also observe simple regret as shown in Figure \ref{fig:simple_regret} to analyze if the estimate of good arms improves over time as the surrogate better learns to estimate the means (utility) of the arms. For \name{}, simple regret is calculated as loss of the ranked subsets in $U_t$ with respect to their true ranking based on LLM feedback based rewards. For CASE, it is RMSE between optimal LLM reward and predicted empirical means owing to it's linear modeling of rewards. We observe that the simple regret of the set $U_t$ - the current estimate of top-m arms decreases gradually. The gap and simple regret for other datasets are reported in \textbf{Appendix \ref{app:gap_regret}}

\subsection{Performance Comparison for Example Subsets Selection }

To answer \textbf{RQ2}, we compare \name{} with static (task-level) and learning to rank based dynamic example selection approaches as shown in Table \ref{tab:ltr-dynamic-perf}. Since \name{} outputs top-$m$ exemplar subsets as part of auxiliary objective, we compare $\name{}_{static}$ with other task-level selection based inference approaches. We observe that $\name{}_{static}$  outperforms existing approaches including CASE across datasets. We hypothesize that this is primarily due to the parameterized non-linear surrogate that models the arm feature-rewards relationship better than CASE and existing approaches.  We also employ the non-linear ranking surrogate trained during the selection to rank example subsets dynamically for each test instance with results indicated by $\name{}_{dynamic}$ in Table \ref{tab:ltr-dynamic-perf}. Firstly we observe  that $\name{}_{dynamic}$ outperforms static example selection approaches demonstrating need for instance-level selection as static set of examples may not generalize to unseen queries. Also from the table, we observe that  $\name{}_{dynamic}$ significantly outperforms existing approaches like KNN and MMR which aim to retrieve examples based on similarity and diversity to the test example respectively. This is primarily because, $\name{}_{dynamic}$ learns to rank subsets by taking into consideration the impact of a particular combination of examples on downstream LLM performance through the training process in the bandit optimization step. Whereas KNN and MMR retrieve examples independently without considering how they may interact together and affect LLM performance. $\name{}_{dynamic}$ also outperforms dynamic version of the CASE, which employs a bandit-based selection algorithm per instance and scores subsets as a whole. For instance, $\name{}_{dynamic}$ achieves upto \textbf{15.99\%} over dynamic CASE on AquaRAT. We observe that the improvements are from better modeling of reward structure and auxiliary objective leading to clear separation of borderline and top-$m$ arms. 

Comparing to LTR approaches, from Table \ref{tab:ltr-dynamic-perf}, we observe that $\name{}_{dynamic}$ outperforms existing LTR approaches trained with different objectives. The LTR approaches are also trained to rank example subsets as a whole for fair comparison and demonstrate gains over static selection approaches. However, we observe that the mechanism to iteratively fit the non-linear surrogate in our approach helps the model clearly distinguish between borderline arms and top-$m$ arms through sampling of most ambiguous arms and reduction of the gap between them. However, in classical LTR training approaches, there is no principled mechanism to sample ambiguous borderline arms and only adopt heuristic negative sampling without considering downstream task performance, leading to subpar performance: ablation in Table \ref{tab:ltr-dynamic-perf}. We also perform an ablation where we replace the proposed gap-index based exploration with random exploration (\namedyn{} (random)). The performance comparison demonstrates that our proposed gap-index exploration leads to better performance compared to random exploration. A qualitative analysis of selected ICL demonstrations is discussed in Appendix \ref{sec:qualitative}.

\begin{table}[h]
\small
\centering

\begin{tabular}{lccc}
\toprule
\textbf{Total \# LLM calls} & \textbf{GSM8K} & \textbf{AquaRat} & \textbf{WMT19} \\
\midrule
CASE & 10200 & 4360 & 87 \\
\name{}$_{\text{static}}$ & \textbf{2600} & \textbf{1580} & \textbf{48} \\
\bottomrule
\end{tabular}
\caption{Results comparing the number of LLM calls.}
\label{tab:llm_calls}
\end{table}
\subsection{Efficiency Comparison}
\label{sec:efficiency}

To answer \textbf{RQ3}, we compare the convergence time of CASE and \name{} (Figure \ref{fig:runtime_comparison}), inference time across the three datasets. We observe that convergence time of \name{} is less than CASE (providing approximately \textbf{2x} speedup on AquaRAT and \textbf{2.35 x} speedup on GSM8K) for identifying top-$m$ arms. We observe that this is primarily due to \textbf{faster convergence} of \name{} based on stopping criterion. This is because the non-linear ranking-based surrogate in \name{}  quickly learns to distinguish between top-$m$ arms and borderline challenger arms by sampling better ambiguous arms through auxiliary objective. This is also evident from comparison of gap-index across rounds between CASE and \name{} in Figure \ref{fig:gap_index}. For instance, in AquaRAT, \name{} converges in \emph{79 rounds}, whereas CASE requires \emph{238 rounds} and for GSM8k, \name{} requires \emph{130} rounds but CASE requires \emph{510} rounds. We also observe the sample-efficiency of \name{} from Table \ref{tab:llm_calls}, where \name{} significantly reduces the number of LLM calls (by \textbf{63-75\%} approx.) during the top-m selection using an auxiliary objective (offline) compared to prior approaches like CASE.

We also plot the average inference times per query on test set. Particularly, we compare \namedyn{} with static CASE to measure the latency overhead of surrogate ranker at inference time in $\name{}_{dynamic}$. We observe from Figure \ref{fig:inference_runtime_comparison}, that the latency overhead is negligible with $\name{}_{dynamic}$, only a few milliseconds over static CASE. Additionally, we also compare the average runtime/query of \namedyn{} with other approaches (Table \ref{tab:inference_time}).  $CASE_{dynamic}$ applies bandit selection at test time, resulting in high runtime. In contrast, \namedyn{}
 is more efficient due to its lightweight ranker, avoiding extra inference overhead unlike NeuralNDCG. PiRank has similar runtime (similar architecture) but performs worse. All static methods (EXPLORA, etc.) have same runtime as CASE (Figure \ref{fig:inference_runtime_comparison})


\section{Conclusion}

In this work, we propose a sample-efficient gap-index based MAB framework (\name{}) as auxiliary objective to fit a non-linear surrogate that models example subset scores to enhance ICL. \name{} provides a mechanism to learn ranking surrogates from LLM feedback for instance-level example subsets selection.  \name{} converges faster than existing MAB algorithms. It adds  only negligible overhead during inference for instance-level setting. The proposed algorithm can also be extended to other ranking tasks in the future with appropriate loss function changes. 

\section{Acknowledgments}
The experiments were enabled by resources provided by the National Academic Infrastructure for Supercomputing in Sweden (NAISS), partially funded by the Swedish Research Council through grant agreement no. 2022-06725 (proposal NAISS 2025/5-631) and Zeus cluster provided by the Department of Computer and Systems Sciences at Stockholm University.

\section{Limitations}
While our approach proposes a novel bandit framework for enhancing In-context learning, accounting for complex non-linear reward relationships, it has a few limitations. Firstly, the task relies on supervised question-answer pairs which may not be available for all domains. However, the feature design in our current algorithm can be easily modified to first select a subset of questions and augment them with answers using LLMs or other data augmentation mechanisms. Since this is not in the scope of the current work, we defer this for future work. Additionally, while we provide sample-complexity bounds as our core objective is to reduce the number of LLM calls for efficiency, regret bounds could also be useful. While we do currently demonstrate empirically convergence of regret and gap-index, theoretical regret bounds are hard to derive for gap-index frameworks. We would refer to existing gap-index frameworks for linear bandits \cite{reda2021top}, which also do not focus on regret bounds and only on sample-complexity. Since we propose a new gap-index framework for non-linear stochastic settings, this required the development of a new theoretical framework for sample-complexity, as shown in the paper, and hence our current work does not focus on regret bounds. However, we believe that our algorithm is useful for settings beyond ICL and can be extended by the community with regret bounds and new advances. Additionally, our framework could be extended to other relaxed versions of learning-to-rank losses, which is also an avenue for further research.

\section{Ethical Considerations and Risks}

We use only publicly available
math reasoning/machine translation datasets that do not contain
private or harmful information. While we employ
LLMs for QA systems in our experiments, we do
not prompt them in a format that would elicit harmful or biased information.




\bibliography{custom}

\appendix

\clearpage
\newpage

\appendix

\section{Appendix}
\begin{table*}[hbt!] 

\begin{tabular}{c|c|c|c}
\multicolumn{1}{l|}{\textbf{Framework}} & \textbf{Loss function}                                                      & \textbf{Hyperparameters}                                                             & \textbf{Network architecture}   \\ \hline
\textbf{PiRank}                         & PiRank surrogate loss                                                       & Presented in Section 4                                                                                             & (256, 256,128,64)               \\ \hline
{Other LTR baselines}                & Neural NDCG                                                                 & \begin{tabular}[c]{@{}c@{}}N = 2, \\ $d_{ff}$ = 384, \\ h = 1, \\ dropout = 0.1\end{tabular}  & (768,96)                   \\ \cline{2-4} 
                                        & ListNet                                                                     & \begin{tabular}[c]{@{}c@{}}N = 4, \\ $d_{ff}$ = 512, \\ h = 2, \\ dropout = 0.3\end{tabular}  & (768,128)                  \\ \cline{2-4} 
                                        & Neural NDCG                                                                 & \begin{tabular}[c]{@{}c@{}}N = 4, \\ $d_{ff}$ = 512, \\ h = 4, \\ dropout  = 0.3\end{tabular} & (768,96)                   \\ \cline{2-4} 
                                        & \begin{tabular}[c]{@{}c@{}}Neural NDCG With \\ Normalized data\end{tabular} & \begin{tabular}[c]{@{}c@{}}N = 2, \\ $d_{ff}$ = 384, \\ h = 1, \\ dropout = 0.1\end{tabular}  & (768,96)                   \\ \cline{2-4} 
                                        & NDCGLoss 2++                                                                & -                                                                                             & (256, 512, 1024, 512, 256)
\end{tabular}
\caption{Details of hyperparameters used in different LTR model configurations. Categorized by loss function and framework.}
\label{tab:ltr_params}
\end{table*}

\begin{table}[h]
\small
\begin{tabular}{p{3cm}ccc}
\toprule
\textbf{Task performance} & \textbf{GSM8K} & \textbf{AquaRat} & \textbf{WMT19} \\
\midrule
NDCGLoss 2++ & 72.11 & 49.00 & 74.06 \\
NeuralNDCG & 69.83 & 48.30 & 72.75 \\
NeuralNDCG (norm. data) & 72.06 & 49.21 & 72.97 \\
UDR & 71.63 & 48.03 & 72.65 \\
 ListNet & 71.79 & 46.06 & 73.64 \\
\textit{DEARICL}$_{dynamic}$ & \textbf{75.66} & \textbf{55.11} & \textbf{78.89} \\
\bottomrule
\end{tabular}
\caption{Ablation results: Training  LTR approaches on training data collected across rounds by \name{} leveraging llama3.2:3b}
\label{tab:data_ablation}
\end{table}
\begin{table*}[t]

    \centering

    \begin{tabular}{llll}

    \toprule
     \textbf{Method}& \multicolumn{1}{l}{\textbf{GSM8K}}& \multicolumn{1}{l}{\textbf{AquaRat}} & \multicolumn{1}{l}{\textbf{WMT}}  

     \\

    \midrule

         \colorg \textbf{Task level} & \colorg & \colorg & \colorg  \\
    
        EXPLORA \cite{purohit-etal-2024-explora} & 93.63 & 69.29 & 84.55  \\
        
 LENS \citep{li2023finding} &76.19 &64.56 & 83.57  \\
           Static CASE &  91.13  & 73.23  & 83.49  \\
  $\name{}_{static}$ & 94.84 & 77.16 & 86.25  \\

       \colorg \textbf{Instance level} & \colorg & \colorg & \colorg  \\


 

$CASE_{dynamic}$ \cite{case}   & 92.19              &  76.77   & 86.36     \\


    \textbf{\namedyn{}} & 94.66 & \textbf{81.88} $\dagger$ & \textbf{92.41}$\dagger$  \\



     \bottomrule
    \end{tabular}

    \caption{Results across datasets (we use 5-shot for all methods) using gpt-4o-mini. $\dagger$ indicates statistical significance (t-test) over $CASE_{dynamic}$ at 0.05 level 
    }
    \label{tab:main_result}

\end{table*}


\begin{table*}[t]

    \centering

    \begin{tabular}{llll}

    \toprule
     \textbf{Method}& \multicolumn{1}{l}{\textbf{GSM8K}}& \multicolumn{1}{l}{\textbf{AquaRat}} & \multicolumn{1}{l}{\textbf{WMT}}  

     \\

    \midrule

         \colorg \textbf{Task level} & \colorg & \colorg & \colorg  \\
    
 \color{black}
        EXPLORA \cite{purohit-etal-2024-explora} &  82.63 & 68.10 & 78.59  \\
    \color{black}     
 LENS \citep{li2023finding} & 77.33 & 57.87 & 76.98  \\
  \color{black}
           Static CASE &  83.09  &  69.29 & 78.26  \\
            \color{black}
  $\name{}_{static}$ & 86.12 & 70.47 & 79.27   \\

       \colorg \textbf{Instance level} & \colorg & \colorg & \colorg  \\


 
 \color{darkblue}
$CASE_{dynamic}$ \cite{case}   & 85.98             &  69.68   &  79.78    \\


 \color{darkblue}
    \textbf{\namedyn{}} & \textbf{88.40} $\dagger$ & \textbf{74.41} $\dagger$ & \textbf{81.50}  \\



     \bottomrule
    \end{tabular}
    
    \caption{\textcolor{black}{Results across datasets using Deepseek-R1:7b (we use 5-shot for all methods). $\dagger$ indicates statistical significance (t-test) over $CASE_{dynamic}$ at 0.05 level }}
    \label{tab:deepseek}

\end{table*}
\color{black}

\section{Dataset Description}
\label{sec:datasets}
  \textbf{AquaRAT}: It comprises 100,000 algebraic word problems in the train set with dev and test set each comprising 254 problems. The problems are provided along with answers and rationales providing the step-by-step solution to the problem.

          \textbf{GSM8K}: This dataset consists of linguistically diverse math problems that require multi-step reasoning. The dataset consists of 8.5K problems and we evaluate on the test set of 1319 questions.

          \textbf{WMT 19}: We focus on en-zh (english-chinese) translation split. Test sets are a few thousand sentences (for example, 3,981 in WMT18 for zh-en direction for test). Train set has 173k english-chinese sentence pairs.

\section{LTR Hyperparameters}

The hyperparameters for learning to rank baselines are detailed in Table \ref{tab:ltr_params}.

\begin{figure*}[!t]
\small
\begin{subfigure}{0.25\linewidth}
\begin{tikzpicture}
\begin{axis}[
    ylabel style = {font= \tiny},
    xlabel style = {font= \small},
    xlabel=Rounds,
    xticklabel style={font=\small},
    yticklabel style={font=\boldmath \tiny},
    ylabel=Gap Index,
    height=4.3cm,
    legend style={
                    font= \tiny,
                },
    width=4cm,
    xmin=1, xmax=120,
    ymin=0, ymax=1.5]
    ytick={0,0.1,0.2,0.3}
    xtick={1,10,20,30,40...,80,100,120}
\addplot[smooth,mark=*,red] plot coordinates {
    (1,1.01)
    (20,0.40)
    (40,0.36)
    (60,0.27)
    (80,0.22)
    (100,0.16)
    (120,0.08)};
\addlegendentry{\name{}}
\addplot[smooth,mark=*,blue] plot coordinates {
   (1,0.95)
    (20,0.44)
    (40,0.22)
    (60,0.15)
    (80,0.14)
    (100,0.12)
    (120,0.12)};
\addlegendentry{CASE}
\end{axis}
\end{tikzpicture}
\caption{Gap Index \\Comparison (GSM8K)}
\label{fig:gap_index_gsm8k}
\end{subfigure}
\hspace{-2em}
\begin{subfigure}{0.3\linewidth}
\begin{tikzpicture}
\begin{axis}[
    ylabel style = {font= \tiny},
    xlabel style = {font= \small},
    xlabel=Rounds,
    xticklabel style={font=\small},
    yticklabel style={font=\boldmath \tiny},
    ylabel=Gap Index,
    height=4.3cm,
    legend style={
                    font= \tiny,
                },
    width=4cm,
    xmin=1, xmax=30,
    ymin=0, ymax=2.5]
    ytick={0,0.1,0.2,0.3}
    xtick={1,10,20,30,40...,80,100,120}
\addplot[smooth,mark=*,red] plot coordinates {
    (1,2.48)
    (5,1.75)
    (10,0.97)
    (15,0.57)
    (20,0.51)
    (25,0.32)
    (29,0.07)};
\addlegendentry{\name{}}
\addplot[smooth,mark=*,blue] plot coordinates {
   (1,0.62)
    (5,0.57)
    (10,0.51)
    (15,0.49)
    (20,0.41)
    (25,0.37)
    (30,0.31)};
\addlegendentry{CASE}
\end{axis}
\end{tikzpicture}

\caption{Gap Index \\Comparison (WMT19)}
\label{fig:gap_index_wmt}
\end{subfigure}
\hspace{-2em}
\begin{subfigure}{0.25\linewidth}
\begin{tikzpicture}
\begin{axis}[
    ylabel style = {font= \tiny},
    xlabel style = {font= \small},
    xlabel=Rounds,
    xticklabel style={font=\small},
    yticklabel style={font=\boldmath \tiny},
    ylabel=Simple Regret,
    height=4.3cm,
    legend style={
                    font= \tiny,
                },
    width=4cm,
    xmin=1, xmax=140,
    ymin=0, ymax=1]
    ytick={0,0.1,0.2,0.3}
    xtick={1,10,20,30,40...,80}
\addplot[smooth,mark=*,red] plot coordinates {
    (1,0.65)
    (20,0.59)
    (40,0.57)
    (60,0.48)
    (80,0.31)
    (100,0.22)
    (120,0.13)
    (130,0.08)};
\addlegendentry{\name{}}
\addplot[smooth,mark=*,blue] plot coordinates {
    (1,0.75)
    (20,0.44)
    (40,0.31)
    (60,0.28)
    (80,0.21)
    (100,0.22)
    (120,0.13)
    (130,0.12)};
\addlegendentry{CASE}
\end{axis}
    \end{tikzpicture}
\centering
\caption{Simple Regret \\ Comparison (GSM8K)}
\label{fig:simple_regret_gsm8k}
\end{subfigure}
\begin{subfigure}{0.25\linewidth}
\begin{tikzpicture}
\begin{axis}[
    ylabel style = {font= \tiny},
    xlabel style = {font= \small},
    xlabel=Rounds,
    xticklabel style={font=\small},
    yticklabel style={font=\boldmath \tiny},
    ylabel=Simple Regret,
    height=4.3cm,
    legend style={
                    font= \tiny,
                },
    width=4cm,
    xmin=1, xmax=30,
    ymin=0, ymax=1]
    ytick={0,0.1,0.2,0.3}
    xtick={1,10,20,30,40...,80}
\addplot[smooth,mark=*,red] plot coordinates {
    (1,0.43)
    (5,0.43)
    (10,0.42)
    (15,0.34)
    (20,0.23)
    (25,0.14)
    (29,0.04)};
\addlegendentry{\name{}}
\addplot[smooth,mark=*,blue] plot coordinates {
    (1,0.56)
    (5,0.50)
    (10,0.48)
    (15,0.34)
    (20,0.26)
    (25,0.26)
    (29,0.25)};
\addlegendentry{CASE}
\end{axis}
    \end{tikzpicture}
\centering
\caption{Simple Regret \\ Comparison (WMT19)}
\label{fig:simple_regret_wmt}
\end{subfigure}
\caption{Top-$m$ arm identification by \name{}, CASE for GSM8k, WMT19.   (a,b) Gap Index ($B_t(ch_t,b_t)$) comparison and (c,d) Simple regret comparison}
\label{fig:regret_gap_gsm}
\end{figure*}
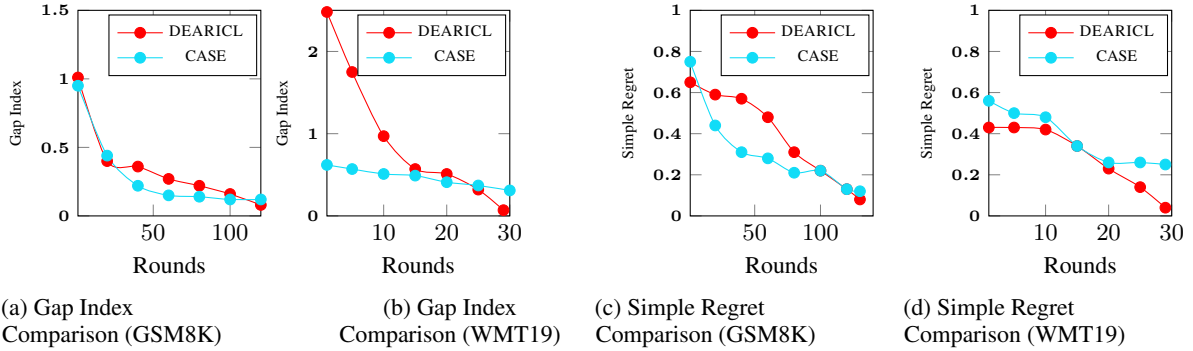

\begin{table*}[h!]
\centering
\small
\begin{threeparttable}

\begin{tabularx}{\textwidth}{@{}lX@{}}
\toprule
\textbf{Question} &
A travel company wants to charter a plane to the Bahamas. Chartering the plane costs \$5,000.
So far, 12 people have signed up for the trip. If the company charges \$200 per ticket, how many
more passengers must sign up for the trip before the company can make any profit on the charter? \\
\midrule
\textbf{Method} & \textbf{Exemplars} \\
\midrule
\textbf{NeuralNDCG} &
\textbf{Question:} A group consists of 4 couples in which each of the 4 boys has one girlfriend. In how many ways can they be arranged in a straight line such that boys and girls occupy alternate positions? \\[0.5ex]
&
\textbf{Question:} A man cycling along the road noticed that every 12 minutes a bus overtakes him and every 4 minutes he meets an oncoming bus. If all buses and the cyclist move at a constant speed, what is the time interval between consecutive buses? \\[0.5ex]
&
\textbf{Question:} Find large number from below question. The difference of two numbers is 1365. On dividing the larger number by the smaller, we get 6 as quotient and the 15 as remainder? \\[0.5ex]
&
\textbf{Question:} 64, 81, 100, 144, 196, ?, Find the missing number(?). \\[0.5ex]
\midrule
\textbf{\namedyn{}} &

\textbf{Question:} A certain sum is invested at simple interest at 18\% p.a. for two years instead of investing at 12\% p.a. for the same time period. Therefore the interest received is more by Rs.~840. Find the sum? \\
&
\textbf{Question:} An investor can sell her MicroTron stock for 36\$ per share and her Dynaco stock for
60\$ per share, If she sells 300 shares altogether, some of each stock, at an average price per share
of 40\$, how many shares of Dynaco stock has she sold? \\[0.5ex]
&
\textbf{Question:} Anna left for city A from city B at 5.20 a.m. She traveled at the speed of 80 km/hr for 2 hrs 15 min. After that the speed was reduced to 60 km/hr. If the distance between two cities is 350 kms, at what time did Anna reach city? \\[0.5ex]
&
\textbf{Question:} On a sum of money, the S.I. for 2 years is Rs.~660, while the C.I. is Rs.~696.30, the rate of interest being the same in both the cases. The rate of interest is? \\[0.5ex]

\bottomrule
\end{tabularx}
\begin{tablenotes}\footnotesize
\item Rationale is not shown to conserve space. However, in our experiments, all exemplars for AquaRat include rationales.
\item The test instance is a Profit, Loss \& Interest problem that requires algebra. Subset selected by \namedyn{} contains three similar questions(1, 2.,4.), whereas NeuralNDCG contains none.
\end{tablenotes}
\end{threeparttable}
\caption{Exemplar subsets selected by  NeuralNDCG and \namedyn{} using LLama3.2:3b for a particular AquaRat test
instance where \namedyn{} had a correct answer and the NeuralNDCG model did not.}
\label{tab:aquarat-exemplars}

\end{table*}
\begin{table}[hbt!]

\begin{tcolorbox}[title= AQUA Prompt]
\small
\textbf{Instruction}:\texttt{You are a helpful, respectful, and honest assistant helping to solve math word problems or tasks requiring reasoning or math. Follow given examples and solve the problems in step by step manner.}

\paragraph{\textbf{Exemplars}}:

[Question]: \textit{ The average age of three boys is 45 years and their ages are in proportion 3:5:7. What is the age in years of the youngest boy?}

[Options]: A) 9, B) 10, C) 11, D) 12, E) 13

[Explanation]: \textsf{$3x + 5x + 7x = 45$, \\ $x =3$, \\ $3x = 9$}

[Answer]: \textcolor{teal}{\textsf{The option is A}}
\\
\dots \\
\dots 
\paragraph{\textbf{Test Input}}: Question: {}
Options: {}

Explanation: [INS]
Answer: [INS]

\end{tcolorbox}
\captionof{figure}{Prompt for Aqua}
\label{prompt:aqua} 
\end{table}

\begin{table}[hbt!]
\begin{tcolorbox}[title= GSM8K Prompt]
\small
\textbf{Instruction}:\texttt{You are a helpful, respectful and honest assistant helping to solve math word problems or tasks requiring reasoning or math. Follow given examples and solve the problems in step by step manner.}

\paragraph{\textbf{Exemplars}}:

[Question]: \textit{ Samir just turned half the age Hania was 10 years ago. If in five years Hania will be 45 years old, what will Samir’s age be five years from now?}

[Explanation]: \textsf{ If in five years, Hania will be 45 years old, currently she is $45 - 5 = 40$ years old.
Samir just turned half the age Hania was 10 years ago, which means she is $30/2 = 15$ years old.
In five years, Samir will be $15 + 5 = 20$ years old.}

[Answer]: \textcolor{teal}{\textsf{20 years old}}
\\
\dots \\
\dots 
\paragraph{\textbf{Test Input}}: Question: {}

Explanation: [INS]
Answer: [INS]

\end{tcolorbox}
\captionof{figure}{Prompt for GSM8K}
\label{prompt:gsm8k}
\end{table}
\section{Gap and simple regret analysis for GSm8K and WMT19}
\label{app:gap_regret}

Similar to the gap index and simple regret analysis on AquaRAT in Section \ref{sec:regret_gap}, we also plot the gap indices and simple regret across rounds for GSM8K and WMT 2019 as shown in Figure \ref{fig:regret_gap_gsm}. We observe a similar trend where gap index between ambiguous arms approaches $\epsilon$, the stopping criterion and the regret also minimizes across rounds demonstrating the empirical convergence of \name{}.

\section{Results on alternative LLMs}
\label{app:alternative_llm}
We also evaluate static and dynamic versions of our proposed approach $\name{}$ using alternative LLMs like gpt-4o-mini and Deepseek-R1:7B (DeepSeek-R1-Distill-Qwen-7B). We choose these LLMs as it has shown relatively stable performance across benchmarks. We extract the most competitive baselines and method from Table \ref{tab:ltr-dynamic-perf} and evaluate them using gpt-4o-mini and Deepseek-R1:7B (DeepSeek-R1-Distill-Qwen-7B). All our experiments are carried out in a transfer setting where exemplars selected using Llama3.2:3b in the optimization loop 
are employed directly for inference on test set using gpt-4o-mini or Deepseek. The results using gpt-4o-mini are as shown in Table \ref{tab:main_result} and results for and Deepseek-R1:7B (DeepSeek-R1-Distill-Qwen-7B) are shown in Table \ref{tab:deepseek}.

\section{Data ablation and extension to other LTR loss functions}
We also perform data ablations where we train the LTR approaches on training data collected across rounds by \name{} to isolate the effect of our auxiliary objective tailored to our setup. The results are shown in Table \ref{tab:data_ablation}. We observe that when compared to Table \ref{tab:ltr-dynamic-perf}, the performance of some baselines show improvement on some benchmarks. But they still fall behind \namedyn{}.

Though training on quality subsets in logged data from \namedyn{} helps improve performance, they miss the benefit of the online adaptive MAB outer loop, optimizing for auxiliary objective. Note that our proposed learning process uses utility from the surrogate, which is also continuously updated from feedback and based on a specific ranking loss. Hence, the borderline arms are adaptively sampled based on the gap-index, which depends on the estimated utility from surrogate. 

Our framework is agnostic to the surrogate architecture: any non-linear scorer can be substituted in the inner loop. We use a lightweight architecture because the ranker is invoked per test instance at inference, where it delivers a runtime advantage over per-instance bandit search. Our choice of lightweight architecture additionally keeps the network's Lipschitz constant controlled. The one component that interacts with our theory is the ranking loss. Our sample-complexity analysis is derived via uniform stability, which requires the inner-loop surrogate loss to be 
$L$-Lipschitz ; Lemma \ref{lemma:lipschitz} establishes this property for our relaxed-NDCG formulation. A stability-controlled loss enables reliable learning from the small, adaptively collected sample that the gap-index acquisition produces. Other LTR objectives can be substituted in the inner loop empirically; however, carrying the sample-complexity related guarantees over requires establishing the analogous constants for each loss (beyond scope of this work and is reserved for future work), which is why our analysis is scoped to relaxed-NDCG, and extending Lemma \ref{lemma:lipschitz}'s analysis to further losses is a natural follow-up work.

\section{Dataset prompts}

The prompts are given in Figures \ref{prompt:aqua}, \ref{prompt:gsm8k}. The prompts for WMT19 are in the github repo.

\section{Qualitative Analysis}
\label{sec:qualitative}
We also analyzed the quality of a few demonstration samples selected for the test instances. One such example is shown in Table \ref{tab:aquarat-exemplars}. The remaining selected ICL samples are provided in the GitHub repo.  The test instance is a Profit, Loss \& Interest problem that requires algebra. Subset selected by \namedyn{} contains three ICL samples (1, 2,4) that contain the required skills to solve the test instance, whereas NeuralNDCG contains none.

\section{Proof for Theorem \ref{theorem1}}
\label{sec:grass_proof}
\paragraph{Notation and setup.}
Let $\mathcal{A}$ be the arm set, $|\mathcal{A}|=K$. At round $t$ the PiRank surrogate (with dropout)
outputs $N$ stochastic predictions (N MC-dropout forward passes) per arm $a$:
\[
\{y_a^{(k)}\}_{k=1}^N,\qquad \hat\rho_t(a):=\frac{1}{N}\sum_{k=1}^N y_a^{(k)}.
\]
Define the sample mean difference and sample variance for pair $(i,j)$:
\begin{align*}
\bar d_{ij}:=\frac{1}{N}\sum_{k=1}^N\big(y_i^{(k)}-y_j^{(k)}\big),\\
\widehat V_{ij,t}:=\frac{1}{N}\sum_{k=1}^N\big(y_i^{(k)}-y_j^{(k)}-\bar d_{ij}\big)^2.
\end{align*}
Let $c_t$ be an annealed multiplier (user-specified) and define
\begin{align*}
B_t(ch_t,b_t):=\hat\rho_t(ch_t)-\hat\rho_t(b_t)+W_t(b_t,ch_t).
\end{align*}
In order to prove Lemma \ref{lemma:m-GRASS_bound}, we first need to demonstrate that event $\cE$ hold with probability $\ge 1-\delta$ for \name{}. However, the estimated means $\hat{\rho_t}$ and $W_t(i,j)$ for any two arms $i,j$ is computed in a  different manner for \name{} than existing linear stochastic bandit frameworks like CASE and GIFA. Hence, we establish that the confidence event $\cE$ holds here by deriving an upper bound on pairwise-gap error as explained below. This is one of our main \textbf{contributions} which further helps in deriving a high complexity upper bound on sample complexity.

To recap the event is defined as,
$\cE \triangleq \bigcap_{t > 0} \bigcap_{i,j \in \ARMS} \Big(\rho_i-\rho_j \in [-B_t(j,i), B_t(i,j)]\Big),$
Expanding the event $\cE$

\[\rho_i-\rho_j \ge (\hat \rho_i(i)-\hat \rho_i(j)) - W_t(i,J)  \]

and symmetrically, 

\[\rho_i-\rho_j \le (\hat \rho_i(i)-\hat \rho_i(j)) + W_t(i,j)  \]

Hence it follows,
\[(\hat \rho_i(i)-\hat \rho_i(j)) - (\rho_i-\rho_j) \le W_t(i,j)\]

Let 
\[
\mathbb{E}_{ij}(t)
:= \bigl(\hat\rho_t(i) - \hat\rho_t(j)\bigr) - \bigl(\rho(i) - \rho(j)\bigr)
\]

denote the pairwise-gap error. This error can be further decomposed as follows based on the source of randomness / errors

\begin{align*}
\mathbb{E}_{ij}(t)
&= \underbrace{\begin{aligned}\bigl[\bigl(\hat\rho_t(i) - \hat\rho_t(j)\bigr) 
 \\- \bigl(\mathbb{E}_{\text{drop}}[y_i] - \mathbb{E}_{\text{drop}}[y_j]\bigr)\bigr]\end{aligned}}_{\textbf{(Term 1) MC Dropout  noise}}\\
& \hspace{-3em}  + \underbrace{\begin{aligned}\bigl[\bigl(\mathbb{E}_{\text{drop}}[y_i] -  \mathbb{E}_{\text{weights,drop}}[y_i]\bigr)
- \\ \hspace{-1em} \bigl(\mathbb{E}_{\text{drop}}[y_j] - \mathbb{E}_{\text{weights,drop}}[y_j]\bigr)\bigr]\end{aligned}}_{\textbf{(Term 2) Weight randomness / SGD-induced drift}} \\
& \hspace{-3em} +\underbrace{\begin{aligned}[t]\bigl[\bigl(\mathbb{E}_{\text{weights,drop}}[y_i]- \rho(i)\bigr) 
\\  -\bigl(\mathbb{E}_{\text{weights,drop}}[y_j] -\rho(j)\bigr)\bigr]\end{aligned}}_{\textbf{(Term 3) Model bias}}
\end{align*}
Each of the above terms can be bounded individually to prove the following theorem.
Restating Theorem \ref{theorem1},

\paragraph{Theorem (MC-Dropout gap concentration).}
In a fixed-confidence setting,  $\delta \in (0,1)$, with probability at least $1-\delta$, 
for all pairs $i,j \in \mathcal{A}$:
\[
\begin{aligned}
\big| (\hat{\rho}_t(i) - \hat{\rho}_t(j)) - (\rho(i)-\rho(j)) \big|
\\ \le c_t \sqrt{ \tfrac{2 \widehat V_{ij,t} \log(2K^2/\delta)}{N} }
+ \varepsilon^{\text{stab}}_t + b_i + b_j 
\\ \quad + \tfrac{4M \log(2K^2/\delta)}{3N}.
\end{aligned}
\]
where $\widehat V_{ij,t}$ is the empirical variance of MC-dropout differences,
$\varepsilon^{\mathrm{stab}}_t$ is the SGD stability error, and $b_i,b_j$ are
surrogate approximation biases.

\textbf{Term 1 - Monte-Carlo Dropout Noise}

For arms i,j we define Monte Carlo Dropout samples of their differences as:

\[Z_k = (y_i(k)-y_j(k)) - (\mathbb{E}_{\text{drop}}[y_i]-\mathbb{E}_{\text{drop}}[y_j]) \]

Our goal is to bound the empirical mean
$\hat Z=\frac{1}{N}\sum_{k=1}^N Z_{k}$ which is equivalent to Term 1.

Assuming $Z_k$'s are independent ( as dropout masks are independent) and bounded, we apply Bernstein's inequality as it states that,

For $\{Z_k\}_{k=1}^{N}$ be independent, mean-zero random variables with 
$|Z_k|\le b$ almost surely, and let 
\[
\hat{Z}_N := \frac{1}{N}\sum_{k=1}^{N} Z_k, 
\qquad 
\sigma^2 := \mathrm{Var}(Z_k).
\]
Then, for any $\epsilon > 0$, the (scalar) Bernstein inequality states:
if $Z_1,\dots,Z_N$ are independent,
mean-zero, and satisfy $|Z_k|\le b$, then for any $\epsilon>0$,
\[
\Pr\!\Bigl(\bigl|\bar{Z}\bigr| \ge \epsilon\Bigr) 
\;\le\;
2 \exp\!\!\left(
-\frac{N \epsilon^2}{2\sigma^2 + \frac{2}{3}b\epsilon}
\right).
\]

Here $b=2M$ and $\sigma^2=\mathrm{Var}(y_i - y_j)$.

To get a confidence radius $\epsilon$
 such that the event holds with probability at least $1-\delta$ we set
\[
2 \exp\!\!\left(
-\frac{N \epsilon^2}{2\sigma^2 + \frac{2}{3}b\epsilon}
\right) = \delta.
\]

This yields the inequality:

\[
2   \left(
-\frac{N \epsilon^2}{2\sigma^2 + \frac{2}{3}b\epsilon}
\right) \ge log \delta.
\]

hence,

\[
   \left(
\frac{N \epsilon^2}{2\sigma^2 + \frac{2}{3}b\epsilon}
\right) \ge log (2/\delta).
\]

Following \cite{Maurer2009EmpiricalBB,lattimore2020bandit} we aim to solve the inequality:

\[
   \left(
\frac{N \epsilon^2}{2\sigma^2 + \frac{2}{3}b\epsilon}
\right) \ge log (2/\delta).
\] 
This can be expressed as a quadratic in $\epsilon$,

\[N \epsilon^2 - \dfrac{2}{3}b\epsilon (log(\dfrac{2}{\delta})) -  2 \sigma^2 log(\dfrac{2}{\delta}) \ge 0  \]

We are interested in the smallest positive \(\varepsilon\) (denoted  $\epsilon^*$) that makes the Bernstein exponent achieve the desired log level. Now,
solving the quadratic equation and using this principle we get,
\begin{equation}
\small
 \epsilon^* =  
\frac{\frac{2}{3} b \log(2/\delta) + 
\sqrt{\left( \frac{2}{3} b \log(2/\delta) \right)^2 + 8 N \sigma^2 \log(2/\delta)}}{2N}.
    \label{label:inequality_epsilon}
\end{equation}

To obtain an upper bound for square root term we use the inequality,

$\sqrt{a^2+x} \le a + \sqrt{x} $, where $a=\left( \frac{2}{3} b \log(2/\delta) \right)$ and $x=8 N \sigma^2 \log(2/\delta)$

hence,
\begin{align*}\sqrt{\left( \frac{2}{3} b \log(2/\delta) \right)^2 + 8 N \sigma^2 \log(2/\delta)} \le \\ \left( \frac{2}{3} b \log(2/\delta) \right) +  \sqrt{8 N \sigma^2 \log(2/\delta)} \end{align*}

Using above in Equation \ref{label:inequality_epsilon}

\[\epsilon^* \le  \frac{\frac{4}{3} b \log(2/\delta)  +  \sqrt{8 N \sigma^2 \log(2/\delta)}} {2N} \]

\[
\epsilon^* \;\le\;
\sqrt{\frac{2\sigma^2 \log (2/\delta)}{N}}
+ \frac{2b \log (2/\delta)}{3N}.
\]

Thus, with probability at least $1-\delta$, (recall that generally $\Pr\!\Bigl(\bigl|\bar{Z}\bigr| \le \epsilon \Bigr) \ge 1-\delta$)
\[
\bigl|\bar{Z}\bigr|
\;\le\;
\sqrt{\frac{2 \sigma^2 \log (2/\delta)}{N}}
+ \frac{4M \log (2/\delta)}{3N}.
\]

 Replacing $\sigma^2$ by the empirical variance:
 
Define the sample variance estimator
\begin{align*}
\widehat{V}_{ij,t} := 
\frac{1}{N}\sum_{k=1}^N \Bigl[
\bigl(y_i^{(k)} - y_j^{(k)}\bigr) - \bar{d}_{ij}
\Bigr]^2, \\
\qquad
\bar{d}_{ij} := \frac{1}{N} \sum_{k=1}^N 
\bigl(y_i^{(k)} - y_j^{(k)}\bigr).
\end{align*}

Then by concentration of empirical variance ( via Bernstein or Bennett bounds),
$\widehat{V}_{ij,t}$ is close to $\sigma^2$ with high probability, so we may plug
$\widehat{V}_{ij,t}$ into the bound:
\[
\bigl|\bar{Z}\bigr|
\;\le\;
{\sqrt{\frac{2 \widehat{V}_{ij,t} \log (2/\delta)}{N}}
+ \frac{4M \log (2/\delta)}{3N}}.
\]
It is to be noted that we interpret the MC-dropout variance as solely as a variance estimator and we derive concentration bounds as standard in bandit analysis.

\textbf{Union bound over all pairs.}
We require the inequality to hold for all pairs $(i,j) \in \mathcal{A}$ simultaneously.
Since there are at most $K^2$ ordered pairs, set
\[
\delta = \frac{\delta}{K^2}.
\]

By a union bound, with probability at least $1-\delta$,
\begin{align*}
\forall i,j \in \mathcal{A}:
\bigl[\bigl(\hat\rho_t(i) - \hat\rho_t(j)\bigr) 
- \bigl(\mathbb{E}_{\text{drop}}[y_i] \\ - \mathbb{E}_{\text{drop}}[y_j]\bigr)|
\\ \le W_t(i,j).
\end{align*}

This defines the desired pairwise confidence width $W_t(i,j)$ under Monte Carlo dropout.



\textbf{Term 2 — Weights Randomness / SGD Stability}

We need to bound

\[
\bigg| \mathbb{E}_{\text{drop}}[y_a] - \mathbb{E}_{\text{weights, drop}}[y_a] \bigg|,
\]

that is, the gap between the conditional dropout mean (given current weights trained on the dataset) and the expectation over randomness in the training set and weights.

To do this it is first essential that the predictions of the differentiable sorting surrogate does not deviate a lot in each round of arm sampling. This translates to proving that for one-epoch SGD the \textit{uniform stability} criterion holds.

This criterion guarantees that a randomized algorithm is uniformly
stable, if for all data sets differing in only one element, the learned models produce nearly the same
predictions. This is applicable to our setup, as in each round after sampling reward from a arm, this new sample (arm+reward) is added to the training set to update the non-linear surrogate with SGD simulating a single epoch of NN training.

For one-epoch SGD, \cite{hardt_lipschitz} show uniform stability bounds of the form

\[
\sup_{z} \big| \ell(\text{SGD}(S), z) - \ell(\text{SGD}(S^{(i)}), z) \big| \leq \varepsilon^{\text{stab}}_t,
\]
, where $S$ and $S^{(i)}$ differ in atmost one data sample and $\ell(\text{SGD}(S^{(i)}), z)$ denotes the loss of the randomized algorithm (Diffsort neural network). 
which can be translated to a bound on predictions. Under our conditions (bounded gradients $G$ and \textbf{Lipschitz loss $L$ - as shown in Lemma \ref{lemma:lipschitz}}), one-epoch SGD has stability that decays with the dataset size and step size; we encapsulate this as $\varepsilon^{\text{stab}}_t$.

Concretely, there exist constants (depending on $G$, $L$, $\eta_t$) such that $\forall a:$

\[
\big| \mathbb{E}_{\text{dropout}}[y_a \mid \text{weights}] - \mathbb{E}_{\text{weights, dropout}}[y_a] \big| \leq \varepsilon^{\text{stab}}_t.
\]

Hence for the pair $(i,j)$, the contribution is at most $2 \varepsilon^{\text{stab}}_t$; we absorb a factor of 2 into the constant and state the theorem with one $\varepsilon^{\text{stab}}_t$ representing the pairwise bound (or keep $+\varepsilon^{\text{stab}}_t$ per side — we used one in the statement for brevity).

\paragraph{Intuition:} Because we train only one epoch per new sample, the model is only mildly unstable: removing or adding one sample cannot arbitrarily change predictions. That bounded change becomes a bias term in the final gap bound.

\textbf{Term 3 -  Model Bias }:

The per arm bias can be defined as:

\[
b_a := \big|\mathbb{E}_{\mathrm{weights,drop}}[y_a] - \rho(a)\big|.
\]

By the triangle inequality,
 \begin{align*}
\big|\mathbb{E}_{\mathrm{weights,drop}}[y_i] - \rho(i)\big| - \\ \big|\mathbb{E}_{\mathrm{weights,drop}}[y_j] - \rho(j)\big|  \\ \le \big|\mathbb{E}_{\mathrm{weights,drop}}[y_i] - \rho(i)\big| + \\ \big|\mathbb{E}_{\mathrm{weights,drop}}[y_j] - \rho(j)\big| 
= b_i + b_j
\end{align*}

Hence,
\begin{align*}\big|\mathbb{E}_{\mathrm{weights,drop}}[y_i] - \rho(i)\big| - \\\big|\mathbb{E}_{\mathrm{weights,drop}}[y_j] - \rho(j)\big| \\ \le b_i + b_j\end{align*}

This implies that the model-bias contribution to the pairwise error is at most the sum of the two per-arm biases. If the surrogate is well-specified $b_a=0$.

Where empirically,  $b_a$ can be computed as the deviation in empirical mean of the arm over the rounds with respect to a moving average of estimated empirical means over past rounds.

Combining bounds for Term 1 ,2 and 3 yields the expression for $W_t$ in Theorem \ref{theorem1}. The event $\cE$ holds for this $W_t$ which is one of our \textbf{main theoretical contributions}. Then the proof for Lemma \ref{lemma:m-GRASS_bound} follows from \cite{reda2021top} which yields a high probability upper bound on sample complexity as stated in Theorem \ref{th:upper_bounds_linear_topm}. We include proof of Lemma \ref{lemma:m-GRASS_bound} in Appendix \ref{sec:lemma_proof} for completion. We then derive the \textbf{high probability upper bound on sample complexity } In Appendix \ref{sample_complexity}.

\begin{lemma}
    Let $\mathrm{NDCG}_\tau(\hat{\boldsymbol{\rho}}, \mathbf{r})$ denote the differentiable (temperature-$\tau$) relaxation of NDCG obtained via a soft permutation matrix $P_\tau(\hat{\boldsymbol{\rho}})$. Assume:
(i) $\|\mathbf{r}\|_\infty \le R_{\max}$,
(ii) the discount function $D(\cdot)$ is bounded by $D_{\max}$,
(iii) $P_\tau$ is $L_P$-Lipschitz with $L_P = O(1/\tau)$.
Then the loss $L(\hat{\boldsymbol{\rho}}, \mathbf{r}) = -\widehat{\mathrm{NDCG}}_\tau(\hat{\boldsymbol{\rho}}, \mathbf{r})$ is $L$-Lipschitz in $\hat{\boldsymbol{\rho}}$ with
\[
L \le \frac{R_{\max} D_{\max}}{Z(\mathbf{r})} \, L_P = O\!\left(\frac{1}{\tau}\right).
\]
\label{lemma:lipschitz}
\end{lemma}

\paragraph{Proof.}
The relaxed NDCG can be written as
\[
\widehat{\mathrm{NDCG}}_\tau(\hat{\boldsymbol{\rho}}, \mathbf{r})
= \frac{1}{Z(\mathbf{r})} \sum_{i,j} P_\tau(i,j;\hat{\boldsymbol{\rho}})\, r_j\, D(i),
\]
which is linear in $P_\tau$. Let $\hat{\boldsymbol{\rho}}_1, \hat{\boldsymbol{\rho}}_2$ be arbitrary. Then
\begin{align*}
&\big|\widehat{\mathrm{NDCG}}_\tau(\hat{\boldsymbol{\rho}}_1, \mathbf{r})
- \widehat{\mathrm{NDCG}}_\tau(\hat{\boldsymbol{\rho}}_2, \mathbf{r})\big| \le \\
&  \frac{1}{Z(\mathbf{r})}
\sum_{i,j} \big|P_\tau(i,j;\hat{\boldsymbol{\rho}}_1) - P_\tau(i,j;\hat{\boldsymbol{\rho}}_2)\big|\, |r_j|\, |D(i)| \\
&\le \frac{R_{\max} D_{\max}}{Z(\mathbf{r})}
\sum_{i,j} \big|P_\tau(i,j;\hat{\boldsymbol{\rho}}_1) - P_\tau(i,j;\hat{\boldsymbol{\rho}}_2)\big| \\
&\le \frac{R_{\max} D_{\max}}{Z(\mathbf{r})}
\, \|P_\tau(\hat{\boldsymbol{\rho}}_1) - P_\tau(\hat{\boldsymbol{\rho}}_2)\|_1 \\
&\le \frac{R_{\max} D_{\max}}{Z(\mathbf{r})} \, L_P \,
\|\hat{\boldsymbol{\rho}}_1 - \hat{\boldsymbol{\rho}}_2\|.
\end{align*}
where $L_P$ denotes the Lipschitz constant of the soft permutation operator $P_\tau$. Thus $L(\hat{\boldsymbol{\rho}}, \mathbf{r}) = -\widehat{\mathrm{NDCG}}_\tau(\hat{\boldsymbol{\rho}}, \mathbf{r})$ is L-Lipschitz.

\section{Sample Complexity Upper Bound}
\subsection{Proof of Lemma \ref{lemma:m-GRASS_bound}}
\label{sec:lemma_proof}

\label{sec:sample_complexity_proof}
\begin{proof}
    We primarily follow the proof structure of GIFA framework \citep{reda2021top} and \cite{case}.

    \textbf{Preliminaries Recap}:
    Let $\TOPM$ be the true set of top-$m$ arms and $(S_m^*)^c$ denote the true set remaining worst arms.  The gap-index between any two arms $i,j$ is computed as: $B_t(i,j)=\hat{\rho}_t(i)-\hat{\rho}_t(j)+W_t(i,j)$.

    , where \begin{align}
        W_t = c_t \sqrt{ \tfrac{2 \widehat V_{ij,t} \log(2K^2/\delta)}{N} }
+ \varepsilon^{\text{stab}}_t + b_i + b_j \notag \\
 + \tfrac{4M \log(2K^2/\delta)}{3N}.
 \label{eq:w_t}
    \end{align}
as derived in the proof for Theorem \ref{theorem1}
    To prove Lemma \ref{lemma:m-GRASS_bound}, we introduce the following property,
    
    \textbf{Property 1}: For $b_t \in U_t$ and $ch_t \in C_t$
    it holds that $\EMPMU{b_t}{t}\ge\EMPMU{ch_t}{t}$.
    Hence, it follows that $B_t(ch_t,b_t) = \EMPGAP{ch_t}{b_t}{t} + W_t(b_t,ch_t) \le W_t(b_t,ch_t)$ as $\EMPGAP{ch_t}{b_t}{t}<0$
From property 1, we can establish that $B_t(ch_t,b_t) \le W_t(b_t,ch_t) $.
Hence, to show that \[B_t(ch_t, b_t)  \leq -({\Delta(b_t)} \lor \Delta(ch_t))+3W_t(b_t,ch_t)\] we consider the following scenarios:
\paragraph{(i)} \textbf{$b_t \in \TOPM$ and $ch_t \notin \TOPM$}: In that case, 
\[\Delta(b_t) = \rho(b_t) - \rho(m+1) ; \Delta(ch_t) = \rho(m) - \rho(ch_t)\]  is the true gap of the arms.

 As event $\cE$ holds from Theorem \ref{theorem1} and Appendix \ref{sec:grass_proof},
 \begin{align*}
 B_t(ch_t,b_t) = -B_t(b_t,ch_t)+2W_t(b_t,ch_t) \\ \leq\Delta(ch_t,b_t)+2W_t(b_t,ch_t)
 \end{align*}
 
 As $ch_t \notin \TOPM$, 
 \[\rho(ch_t) \leq \rho(m+1)\] \[ \Delta(ch_t,b_t) \leq  \rho(m+1) - \rho(b_t) = -{\Delta(b_t)}\] 
 
 But as $b_t \in \TOPM$, it also holds that $\rho(b_t) \geq \rho(m)$, and $\Delta(ch_t,b_t) \leq  \rho(ch_t) - \rho(m) = -{\Delta(ch_t)}$. Hence, 
  
 \begin{align*}
 B_t(ch_t, b_t) \leq -({\Delta(b_t)} \lor {\Delta(ch_t)})+2W_t(b_t,c_t) \\ \leq -(\Delta(b_t) \lor {\Delta(ch_t)})+3W_t(b_t,c_t).
 \end{align*}

\paragraph{(ii)} \textbf{ $b_t \notin \TOPM$ and $ch_t \in \TOPM$ }:
\[ \Delta(ch_t) = \rho(ch_t) - \rho(m+1) ; \]  
\[\Delta(b_t) = \rho(m) - \rho(b_t)\]

By Property 1,
\begin{align*}
    B_t(ch_t,b_t) \le W_t(b_t,ch_t) \\ \le \hat\Delta_t(b_t,ch_t)  + W_t(b_t,ch_t)  = B_t(b_t,ch_t) 
\end{align*}

as $\hat\rho_t(b_t) \ge \hat\rho_t(ch_t)$. Further, as $\cE$ holds, 

\begin{align*}
B_t(b_t,ch_t) = -B_t(ch_t,b_t) +  2W_t(b_t,ch_t)  \\ \le \Delta(b_t,ch_t)+  2W_t(b_t,ch_t)
\end{align*}

As $b_t \notin \TOPM$, $\rho(b_t) \le \rho(m+1)$
and hence $\Delta(b_t,ch_t) \le \rho(m+1)-\rho(ch_t) = -\Delta(ch_t)$
As $ch_t \in \TOPM$, $\rho(ch_t) \ge \rho(m)$
and hence $\Delta(b_t,ch_t) \le \rho(b_t) - \rho(m) = -\Delta(b_t)$. Hence, 

\begin{align*}
B_t(ch_t, b_t) \leq -({\Delta(b_t)} \lor {\Delta(ch_t)})+2W_t(b_t,c_t) \\ \leq -(\Delta(b_t) \lor {\Delta(ch_t)})+3W_t(b_t,c_t).
\end{align*}

\paragraph{(iii)} \textbf{$b_t \notin \TOPM$ and $ch_t \notin \TOPM$}: 
We state that there exists a $b \in \TOPM$ that belongs to $C_t$.  At any time t,
    \[        M_t \leftarrow random \ m' arms from (U_t \cup C_{t-1})^c   \]
       \[ C_t \leftarrow \mbox{top}_{m'}(M_t \cup C_{t-1}; \hat{\rho}_{(t-1)})\]
       Due to the above sampling approach adopted for $C_t$ which captures the next m' arms with the highest means, it follows by definition of estimated means that $C_t$ captures at least one arm in $\TOPM$.
Given that $\cE$ holds and $b \in \TOPM$,
\[W_t(b_t,ch_t) \ge B_t(ch_t,b_t) \ge B_t(b,b_t)\]

$ch_t$ by the definition is one of the most ambiguous arms posing largest threat to $b_t$ as it has the largest gap with respect to $b_t$ $B_t(ch_t,b_t) \ge B_t(b,b_t)$. Hence, $ B_t(ch_t,b_t) \ge B_t(b,b_t)$. From this and event $\cE$ it follows 
\begin{align*}B_t(ch_t,b_t) \ge B_t(b,b_t) \ge \rho(b)- \rho(b_t) \\ \ge \rho(m)-\rho(b_t)\end{align*}.
Hence $ W_t(b_t,ch_t) \ge B_t(ch_t,b_t) \ge \Delta(b_t)$. Using event $\cE$,

\begin{align*}
\ B_t(ch_t,b_t) \le \Delta(ch_t,b_t) + 2W_t(b_t,ch_t) \\ = (\rho(ch_t)-\rho(m))+ \\(\rho(m)-\rho(b_t)) + 2 W_t(b_t,ch_t)
\end{align*}

From above Eq and since $B_t(ch_t,b_t) \ge \Delta(b_t)$, 

\begin{align*}
B_t(ch_t,b_t) \le -\Delta(ch_t) + \Delta(b_t) + 2 W_t(b_t,ch_t) \\ \le -\Delta(ch_t) + 3 W_t(b_t,ch_t)
\end{align*}

Also from Property 1 and $W_t(b_t,ch_t) \ge \Delta(b_t)$, it holds that

\begin{align*}B_t(ch_t,b_t) \le W_t(b_t,ch_t) = - W_t(b_t,ch_t) \\+ 2W_t(b_t,ch_t)  \le -\Delta(b_t) +2 W_t(b_t,ch_t) \\ \le -\Delta(b_t) + 3W_t(b_t,ch_t) 
\end{align*}

Hence $B_t(ch_t, b_t) \leq -(\Delta(b_t) \lor {\Delta(ch_t)})+3W_t(b_t,c_t)$.

\paragraph{(iv)} \textbf{$b_t \in \TOPM$ and $ch_t \in \TOPM$}:
Then there exists a $s \notin S_m^*$ and $s \in U_t$
In that case, 
\begin{align*}\Delta(b_t) = \rho(b_t) - \rho(m+1) ; \Delta(ch_t) \\ = \rho(ch_t) - \rho(m+1)\end{align*}
 
Also by definition of $b_t$ and $ch_t$, it holds that $B_t(ch_t,b_t) = \max_{i\in U_t} \max_{j\in C_t} \left[ B_t(j,i) \right] $
 Since there exists $s \in U_t$ and $ch_t \in C_t$,
 \begin{align*}
 B_t(ch_t,b_t) = \max_{i\in U_t} \max_{j\in C_t} \left[ B_t(j,i) \right]  \\ \ge \max_{j\in C_t}  B_t(j,s)  \ge B_t(ch_t,s)  \\ \ge \rho(ch_t)-\rho(s) \ge \rho(ch_t)-\rho(m+1)
 \end{align*}

 As $\rho(ch_t)-\rho(m+1) = \Delta(ch_t)$,
 $B_t(ch_t,b_t) \ge \Delta(ch_t)$
 By property 1, $B_t(ch_t,b_t) \le W_t(b_t,ch_t)$.
 Hence, \[\Delta(ch_t) \le B_t(ch_t,b_t) \le W_t(b_t,ch_t)\]

 On event $\cE$ it follows that
 $B_t(ch_t,b_t) \le \rho(ch_t)-\rho(b_t) + 2 W_t(b_t,ch_t)$ as $(B(ch_t,b_t) \le W_t(b_t,ch_t)$. Then $\rho(ch_t)-\rho(b_t)$ can be expressed as $\rho(ch_t) - \rho(m+1) + \rho (m+1) -\rho(b_t)$. hence,

 \begin{align*}
      B_t(ch_t,b_t) \le \rho(ch_t) - \rho(m+1)   + \rho (m+1)  \\ -\rho(b_t) + 2W_t(b_t,ch_t)    \\ \le \Delta(ch_t) - \Delta(b_t) + 2W_t(b_t,ch_t)
 \end{align*}
 
 We already know that $B_t(ch_t,b_t) \ge \Delta(ch_t)$ resulting in,
 
\[(a) \  B_t(ch_t,b_t) \le - \Delta(b_t) + 3W_t(b_t,ch_t)\]

 Now to prove $B_t(ch_t,b_t) \le - \Delta(ch_t) + 3W_t(b_t,ch_t)$, we rely on property 1,
 \begin{align*}B(ch_t,b_t) \le W_t(b_t,ch_t) \\ \le -W_t(b_t,ch_t) + 2 W_t(b_t,ch_t)\end{align*}
 As $W_t(b_t,ch_t) \ge \Delta(ch_t)$, $-W_t(b_t,ch_t) \le -\Delta(ch_t)$. Hence,

\begin{align*}
  (b)   B(ch_t,b_t) \le W_t(b_t,ch_t) \le -W_t(b_t,ch_t) \\ + 2 W_t(b_t,ch_t)  \le - \Delta(ch_t) + W_t(b_t,ch_t) \\ \le -\Delta(ch_t) + 3 W_t(b_t,ch_t)
\end{align*}
  
 From (a) and (b)  
 
 \begin{equation}
     B_t(ch_t, b_t) \leq -(\Delta(b_t) \lor {\Delta(ch_t)})+3W_t(b_t,c_t)
     \label{lemma1_formula}
 \end{equation}
\end{proof}

\subsection{Proof Blueprint for Theorem \ref{th:upper_bounds_linear_topm}}
\label{sample_complexity}
\begin{proof}

We now convert \eqref{lemma1_formula} into sampling bounds by using the stopping
rule and the explicit form of \(W_t\). The intuition is similar to Lemma~8 in
GIFA \cite{reda2021top}, where once the stopping rule \(B_t(ch_t,b_t)\le\varepsilon\) triggers, arms
with non-zero gap must have been sampled enough times so that the width is
small relative to the gap. We invert this relation to obtain a per-arm bound.

\paragraph{Stopping rule.}
Assume the algorithm stops when
\[
B_t(ch_t,b_t)\le\varepsilon.
\]
On the event \(\mathcal{E}\), by Lemma~\ref{lemma:m-GRASS_bound} at time t $<$ stopping time we have
\begin{align*}
\varepsilon \le B_t(ch_t,b_t)
\le -(\Delta(b_t)\vee\Delta(ch_t)) \\ + 3W_t(b_t,ch_t).
\end{align*}
Rearrange to get
\[
3W_t(b_t,ch_t) \ge \varepsilon + (\Delta(b_t)\vee\Delta(ch_t)).
\]

\[
W_t(b_t,ch_t) \ge \frac{\epsilon + \Delta_a}{3}
\]

Hence, for any arm \(a\) that remains active (i.e. is sampled further until
elimination), when it is sampled at time \(t\) its associated width at that
time must satisfy the above inequality (with \(a\) playing the role of \(b_t\)
or \(ch_t\) in the identity). Substituting leading term of $W_t(i,j)$ (first term in
Equation \ref{eq:w_t}):
$c_t\sqrt{\frac{2\widehat V_{b_t,ch_t,t}\,log(2 K^2 / \delta)}{N}}$ (we ignore stability and bias terms for clarity and also because they are negligible in sample complexity upper bound derivation)

\[
c_t\sqrt{\frac{2\widehat V_{b_t,ch_t,t}\,log(2 K^2 / \delta)}{N}} \ge \frac{\epsilon + \Delta_a}{3}
\]

Since  the predictive variance decreases with the number of arm samples $n_t(a)$

\[\widehat V_{b_t,ch_t,t} \le \frac{\sigma_{a,t} ^2}{n_t(a)}\]

Where $\sigma_{a,t}^2$ is the effective variance of arm $a$

Substituting this in the inequality from earlier we get,

\[
c_t \sigma_{a,t}\sqrt{\frac{2\,log(2 K^2 / \delta)}{N . n_t(a)}} \ge \frac{\epsilon + \Delta_a}{3}
\]

Taking square on both sides

\[
c_t^2 \sigma_{a,t}^2\frac{2\,log(2 K^2 / \delta)}{N . n_t(a)} \ge \frac{(\epsilon + \Delta_a)^2}{9}
\]

\begin{align*}
n_t(a) \le 18 
c_t^2 \sigma_{a,t}^2\frac{log(2 K^2 / \delta)}{N  {(\epsilon + \Delta_a)^2}}  
\end{align*}

To account for case when $\Delta_a$ is tiny we replace $(\epsilon + \Delta_a)^-2$ yielding

\begin{align*}
n_t(a) \le  18 
c_t^2 \sigma_{a,t}^2\frac{log(2 K^2 / \delta)}{N  } .\\  \max\!\Big\{ \varepsilon^{-2},\; \big(\tfrac{\varepsilon+\Delta_a}{3}\big)^{-2} \Big\} 
\end{align*}




\paragraph{Per-arm and total bounds.}
Formally,  on event
\(\mathcal{E}\), for every arm \(a\),
\begin{align}\label{eq:per-arm-bound}
\mathcal{N}_T(a) \le 18 
c_t^2 \sigma_{a,t}^2\frac{log(2 K^2 / \delta)}{N  } \cdot
\notag\\ \max\!\Big\{ \varepsilon^{-2},\; \big(\tfrac{\varepsilon+\Delta(a)}{3}\big)^{-2} \Big\}.
\end{align}
Summing over arms yields the total-sample upper bound
\begin{align}\label{eq:total-bound}
T \le 18 
c_t^2 \frac{log(2 K^2 / \delta)}{N  } \sum_{a\in\mathcal{A}} \sigma_{a,t}^2\cdot
\notag \\ \max\!\Big\{ \varepsilon^{-2},\; \big(\tfrac{\varepsilon+\Delta(a)}{3}\big)^{-2} \Big\}.
\end{align}

 \end{proof}
The above equation leads to the upper bound on sample complexity as stated in Theorem \ref{th:upper_bounds_linear_topm}.



\qed

\end{document}